%% file: DensityEntailment-lncs.tex
\documentclass[runningheads,a4paper]{llncs}
\usepackage{tikz}
\usepackage{graphicx}  
\usepackage{url}      
\usepackage{times}
\usepackage[square,sort,comma,numbers]{natbib}
\usepackage{amsfonts}
\usepackage{amsmath}
\usepackage{amsfonts}
\usepackage{verbatim}
\usepackage{fixltx2e}
\usepackage{fouridx}

\input{book-header}

\pgfdeclarelayer{edgelayer}
\pgfdeclarelayer{nodelayer}
\pgfsetlayers{edgelayer,nodelayer,main}

\tikzstyle{none}=[inner sep=0pt]
\tikzstyle{box}=[rectangle,fill=White,draw=Black]
\tikzstyle{ket}=[state,hflip,fill=White,draw=Black]
\tikzstyle{bra}=[state,fill=White,draw=Black]

\newcommand{\ra}{\rightarrow}

\newcommand{\R}{\mathbb{R}}

\newcommand{\tns}{\otimes}

\newcommand{\bo}{\mathbf}

\renewcommand{\ts}{\otimes}
\newcommand{\tr}{\text{tr}}

\newcommand{\braket}[2]{\langle #1 \vert #2 \rangle}
\newcommand{\ketbra}[2]{\vert #1 \rangle  \langle #2 \vert}

\renewcommand{\vec}[1]{\overrightarrow{#1}}
\newcommand{\denst}[1]{\widehat{ #1}}
\newcommand{\prj}[1]{\vert \overrightarrow{#1} \rangle  \langle \overrightarrow{#1} \vert }
\newcommand{\prjdif}[2]{\vert \overrightarrow{#1} \rangle  \langle \overrightarrow{#2} \vert }
\newcommand{\ov}[1]{\overrightarrow{#1}}

\begin{document}

\mainmatter  % start of an individual contribution

\title{Distributional Sentence Entailment Using Density Matrices}

\titlerunning{Distributional Sentence Entailment Using Density Matrices}

\author{Esma Balkir\inst{1}  \and Mehrnoosh Sadrzadeh\inst{1} \and Bob Coecke\inst{2}}
\authorrunning{E. Balk{\i}r, M. Sadrzadeh \& B. Coecke}

\institute{Queen Mary University of London \and University of Oxford}


\maketitle

\begin{abstract}
Categorical compositional distributional model of  \cite{coecke2010} suggests a way to combine grammatical composition of the formal, type logical models with the corpus based, empirical word representations of distributional semantics. This paper contributes to the project by expanding the model to also capture entailment relations. This is achieved by extending the representations of words from points in meaning space to density operators, which are probability distributions on the subspaces of the space. A symmetric measure of similarity and an asymmetric measure of entailment is defined, where lexical entailment is measured using von Neumann entropy, the quantum variant of Kullback-Leibler divergence. Lexical entailment, combined with the composition map on word representations, provides a method to obtain entailment relations on the level of sentences. Truth theoretic and corpus-based examples are provided. 
\end{abstract}

\section{Introduction}
The term \textit{distributional semantics} is almost synonymous with the term \textit{vector 
space models of meaning}. This is because vector spaces are  natural
candidates for modelling the \emph{distributional hypothesis} and contextual similarity between words \citep{firth1957}.  In a nutshell, this hypothesis says that words that often occur in the same contexts have similar meanings. So for instance, `ale' and `lager' are similar since they both often occur in the context of `beer', `pub', and `pint'. 
The obvious practicality of these models, however, does not guarantee that they possess the 
expressive power needed to model all aspects of meaning. Current distributional models mostly fall
short of successfully modelling subsumbtion and entailment \citep{lenci2012}. There are a number of models that use distributional similarity to enhance textual entailment  \citep{glickman2005, beltagy2013}.  However, most of the work from the distributional semantics community 
has been focused on developing more sophisticated metrics on vector representations \citep{lin1998, kotlerman2010, weeds2004}.

In this paper we suggest the use of density matrices instead of vector spaces as the basic distributional
representations for the meanings of words. Density matrices are widely used in quantum mechanics,
and are a generalization of vectors. There are several advantages to using density matrices to model meaning. Firstly,  density matrices have the expressive power to represent all the information vectors can represent: 
they are a suitable implementation of the distributional hypothesis.
They come equipped with a measure of information content, and so provide a natural 
way of implementing asymmetric relations between words such as hyponymy-hypernymy relations. Futhermore, they form a compact closed category. This allows the previous work of \cite{coecke2010} on obtaining representations for meanings of sentences from the meaning of words to be applicable to density 
matrices. The categorical map from meanings of words to the meaning of the sentence respects the order induced by the relative entropy of density matrices. This promises, given suitable representations of individual words, a method to obtain entailment relations on the level of sentences, inline with the  lexical entailment of \emph{natural logic}, e.g. see \cite{MacCartney2007}, rather than the traditional  logical entailment of Montague semantics. 

\paragraph{Related Work.}This work builds upon and relates to literature on compositional distributional  models, distributional lexical entailment, and the use of density matrices in computational linguistics and information retrieval. 

There has been a recent interest in methods of composition within the distributional semantics framework. There 
are a number of composition methods in literature. See \cite{grefenstette2013} for a survey of compositional 
distributional models and a discussion of their strengths and weaknesses. This work extends the work presented in \cite{coecke2010}, 
a compositional model based on category theory. Their  model 
was shown to outperform the competing compositional models by \cite{grefenstette2011}.

Research on distributional entailment has mostly been focused on lexical entailment. One notable exception is \cite{baroni2012}, who use the distributional data on adjective-noun pairs to train a classifier, which is then utilized to detect novel noun pairs that have the same relation. 
There are a number of non-symmetric lexical entailment measures \cite{weeds2004, clarke2009,
kotlerman2010, lenci2012} which all rely on some variation of
the \textit{Distributional Inclusion Hypothesis}: ``If $u$ is semantically narrower than $v$, 
then a significant number of salient distributional features of $u$ are also included in the feature vector of $v$ '' \cite{kotlerman2010}. In their experiments, \cite{geffet2005} show that while if a word $v$ entails another word $w$ then the characteristic features of $v$ is a subset of the ones for $w$, it's not necessarily the case that the inclusion of the characteristic features $v$ in $w$ indicate that $v$ entails $w$. One of their suggestions for increasing the prediction power of their method is to
 include more than one word in the features.

\cite{santus2014} use a measure based on entropy to detect hyponym-hypernym relationships in given pairs.
The measure they suggest rely on the hypothesis that hypernyms are semantically more general than 
hyponyms, and therefore tend to occur in less informative contexts.
 \cite{herbelot2013} rely on a very similar idea, and use KL-divergence between the target
 word and the basis words to quantify the semantic content of the target word. They conclude that this 
 method performs equally well in detecting hyponym-hypernym pairs as their baseline prediction method that only considers the overall frequency 
 of the word in corpus. They reject the hypothesis that
more general words occur in less informative contexts. Their method differs from ours 
in that they use relative entropy to quantify the overall information content of a word, and not to compare 
two target words to each other. 

\cite{piedeleu2015} extend the compositional model of \cite{coecke2010} to include density matrices as we do, but use it for modeling homonymy and polysemy. Their approach is complementary to ours, and in fact, they show that it is possible to merge the two constructions. \cite{bruza2005} use density matrices to model
context effects in a conceptual space. In their quantum mechanics inspired model,
words are represented by mixed states and each eigenstate represents a sense of the word. 
Context effects are then modelled as quantum collapse. \cite{blacoe2013} use density matrices to encode dependency neighbourhoods, with the aim of modelling context effects in similarity tasks.  \cite{van2004} uses density matrices to sketch out a theory of information retrieval,  
and 
connects the logic of the space and of density matrices via an order relation that makes 
the set of projectors in a Hilbert space into a complete lattice. He uses this order to define  
an entailment relation. \cite{sordoni2013} show that using density matrices to represent documents provides significant improvement on realistic IR tasks. 

This paper is based on the MSc Thesis of the first author \citep{balkir2014}.

\section{Background}

\begin{definition} 
A monoidal category is \textbf{compact closed} if for any object $A$, there are  left and right
dual objects, i.e. objects $A^r$ and $A^l$, and morphisms $
\eta^l : I \rightarrow A \otimes A^l$, 
$\eta^r: I \rightarrow A^r \otimes A$, 
 $\epsilon^l : A^l \otimes A  \rightarrow I$ and
$\epsilon^r: A \otimes A^r \rightarrow I$ that satisfy:
\begin{center}
$\begin{array}{cc}
(1_A \otimes \epsilon^l ) \circ ( \eta^l \otimes 1_A) = 1_A &
 (\epsilon^r \otimes 1_A ) \circ (1_A \otimes \eta^r ) = 1_A \\
( \epsilon^l \otimes 1_{A^l} ) \circ (1_{A^l} \otimes \eta^l) = 1_{A^l} &
(1_{A^r }\otimes \epsilon^r ) \circ ( \eta^r \otimes 1_{A^r}) = 1_{A^r} 
\end{array}$
\end{center}
\end{definition}

Compact closed categories are used to represent \textit{correlations}, and in categorical quantum mechanics
they model maximally entangled states. \citep{abramsky2004} The $\eta$ and $\epsilon$ maps are useful 
in modeling the interactions of the different parts of a system. To see how this relates to natural language, consider a simple 
sentence with an object, a subject and a transitive verb. The meaning of the entire sentence is not 
simply an accumulation of the meanings of its  individual words, but depends on how the transitive verb relates 
the subject and the object. The $\eta$ and $\epsilon$ maps provide the mathematical 
formalism to specify such interactions. The distinct left and right duals ensure that compact
closed categories can take word order into account.

There is a graphical calculus used to reason about  monoidal categories \citep{coecke2010B}.
In the graphical language, objects are wires, and morphisms are boxes with incoming 
and outgoing wires of types corresponding to the input and output types of the morphism.
The identity object is depicted as empty space, so a state $\psi: I \ra A$ is depicted
as a box with no input wire and an output wire with type $A$. The duals of states are called \textit{effects}, and they are of type $A \ra I$. Let $f: A \ra B$, $g: B \ra C$ and $h : C \ra D$, and $1_A: A \ra A$ the identity function on $A$. $1_A$, $f$,  $f \otimes h$, $g \circ f$ are depicted as follows:
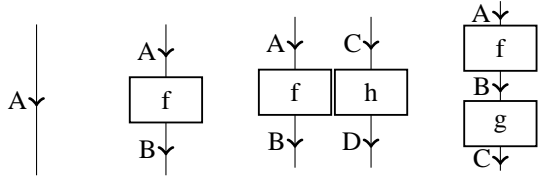
\begin{center}
\vspace{-4mm}
\begin{tikzpicture}
 \begin{pgfonlayer}{nodelayer}
 \node [style=none] (0) at (0,1 ) {};
		\node [style=none] (1) at (0,-1 ) {};
		\node (2) at (1, 0) {};
 \end{pgfonlayer}
\begin{pgfonlayer}{edgelayer}
	\draw [arrow=.5] (0.center) to node[left]{A} (1.center);
 \end{pgfonlayer}
 \end{tikzpicture}
\begin{tikzpicture}
 \begin{pgfonlayer}{nodelayer}
 	\node [morphism] (3) at (0, 0) {f};
	%	\node [morphism] (4) at (-3.25, -0.5) {g};
		\node [style=none] (5) at (0, 1) {};
		\node [style=none] (6) at (0, -1) {};
		\node (7) at (1, 0) {};
 \end{pgfonlayer}
\begin{pgfonlayer}{edgelayer}
	\draw [arrow=.5] (5.center) to node[left]{A} (3.north);
		% \draw [arrow=.5] [in=90, out=-90] (3.south) to node[left]{B} (4.north);
		\draw [arrow=.5] [in=90, out=-90] (3.south) to node[left]{B} (6.center);
 \end{pgfonlayer}
 \end{tikzpicture} 
 \begin{tikzpicture}
 \begin{pgfonlayer}{nodelayer}
 \node (0) at (0,0){};
 \node (1) at (1,0){};
 \node [morphism] (2) at (0,-1) {f};
 \node [morphism] (3) at (1,-1) {h};
 \node (4) at (0,-2){};
 \node (5) at (1,-2){};
 \node (6) at (2,0){};
 \end{pgfonlayer}
\begin{pgfonlayer}{edgelayer}
	\draw [arrow=.5] (0.center) to node[left]{A} (2.north);
	\draw [arrow=.5] (2.south) to node[left]{B} (4.center);
	\draw [arrow=.5] (1.center) to node[left]{C} (3.north);
	\draw [arrow=.5] (3.south) to node[left]{D} (5.center);
 \end{pgfonlayer}
 \end{tikzpicture}
\begin{tikzpicture}
 \begin{pgfonlayer}{nodelayer}
 \node (0) at (0,-.5){};
 \node [morphism] (1) at (0,-1){f};
 \node [morphism] (2) at (0,-2){g};
 \node (3) at (0,-2.5){};
 \node (4) (1,0){};
 \end{pgfonlayer}
\begin{pgfonlayer}{edgelayer}
\draw [arrow=.5] (0.north) to node[left]{A} (1.north);
\draw [arrow=.5] (1.south) to node[left]{B} (2.north);
\draw [arrow=.5] (2.south) to node[left]{C} (3.south);
 \end{pgfonlayer}
 \end{tikzpicture}
 
 \noindent
 The state $\psi: I \ra A$, the effect $\pi: A \ra I$, and the scalar $\psi \circ \pi$ are depicted as follows:
 
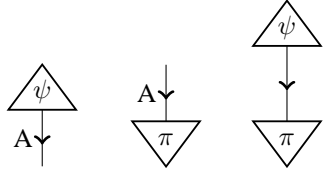
\begin{tikzpicture}
	\begin{pgfonlayer}{nodelayer}
	%	\node [style=none] (0) at (-6, 0.75) {$\psi$:};
		\node[state,hflip] (1) at (-5, 1) {$\psi$};
		\node [style=none] (2) at (-5, 0.25) {};
		\end{pgfonlayer}
	\begin{pgfonlayer}{edgelayer}
	\draw [arrow=.5] (1) to node[left]{A} (2.center);
	\end{pgfonlayer}
	 \end{tikzpicture}
	 \hspace{5mm}
	 \begin{tikzpicture}
	\begin{pgfonlayer}{nodelayer}
		% \node [style=none] (3) at (-3, 0.75) {$\pi$:};
		\node [style=none] (4) at (-2, 1) {};
		\node[state](5) at (-2, 0.25) {$\pi$};
		\end{pgfonlayer}
	\begin{pgfonlayer}{edgelayer}
	\draw [arrow=.5] (4.center) to node[left]{A} (5);
	\end{pgfonlayer}
 \end{tikzpicture}
 \hspace{5mm}
	\begin{tikzpicture}
	\begin{pgfonlayer}{nodelayer}
		\node[state] (6) at (1, 0.25) {$\pi$};
		\node[state,hflip](7) at (1, 1.25) {$\psi$};
	%	\node [style=none] (8) at (0, 0.75) {$\pi \circ \psi$:};
	\end{pgfonlayer}
	\begin{pgfonlayer}{edgelayer}
	\draw [arrow=.5] (7) to node[right]{} (6);
	\end{pgfonlayer}
 \end{tikzpicture}
 \end{center}

The maps $\eta^l, \eta^r, \epsilon^l$ and $\epsilon^r$ take the following forms 
in the graphical calculus:

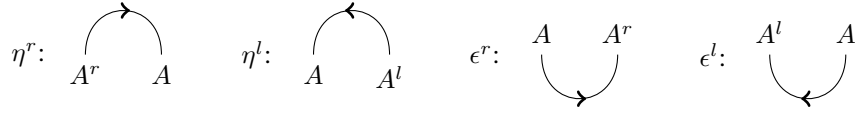
\begin{center}
\begin{tikzpicture}
	\begin{pgfonlayer}{nodelayer}
		\node [style=none, label=below: {$A^r$} ] (0) at (-6, 0){} ;
		\node [style=none, label=below: {$A$} ] (1) at (-5, 0) {};
		\node [style=none, label=below: {$A$} ] (2) at (-3, 0) {};
		\node [style=none, label=below: {$A^l$} ] (3) at (-2, 0){} ;
		\node [style=none, label=above: {$A$} ] (4) at (0, 0) {} ;
		\node [style=none, label=above: {$A^r$}] (5) at (1, 0) {} ;
		\node [style=none, label=above: {$A^l$}] (6) at (3, 0) {} ;
		\node [style=none, label=above: {$A$}] (7) at (4, 0) {} ;
		\node [style=none] (8) at (-6.75, 0) {$\eta^r$:};
		\node [style=none] (9) at (-3.75, 0) {$\eta^l$:};
		\node [style=none] (10) at (-0.75, 0) {$\epsilon^r$:};
		\node [style=none] (11) at (2.25, 0) {$\epsilon^l$:};
	\end{pgfonlayer}
	\begin{pgfonlayer}{edgelayer}
		\draw [arrow=.5] [style=none, bend left=90, looseness=2.00] (0.center) to (1.center);
		\draw [arrow=.5][style=none, bend right=90, looseness=2.00] (3.center) to (2.center);
		\draw [arrow=.5] [style=none, bend right=90, looseness=2.00] (4.center) to (5.center);
		\draw [arrow=.5] [style=none, bend left=90, looseness=2.00] (7.center) to (6.center);
	\end{pgfonlayer}
\end{tikzpicture}
\end{center}

The axioms of compact closure, referred to as the \textit{snake identities} because of the visual form they 
take in the graphical calculus, are represented as follows:

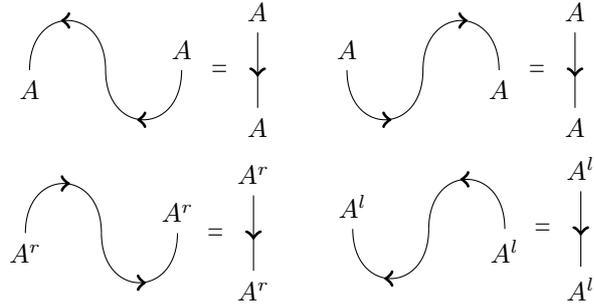
\begin{center}
\begin{tikzpicture}
	\begin{pgfonlayer}{nodelayer}
		\node [style=none, label=below: {$A$} ] (0) at (-6, 1) {};
		\node [style=none] (1) at (-5, 1) {};
		\node [style=none, label=above:{$A$}] (2) at (-4, 1) {};
		\node [style=none, ] (20) at (-3.5, 1) {=};
		\node [style=none, label=above:{$A$}] (12) at (-3, 1.5) {};
		\node [style=none,  label=below:{$A$}] (13) at (-3, 0.5) {};
		\end{pgfonlayer}
\begin{pgfonlayer}{edgelayer}
	\draw [arrow=.5] [bend right=90, looseness=2.25] (1.center) to (0.center);
		\draw[arrow=.5]  [bend left=90, looseness=2.25] (2.center) to (1.center);
	\draw [arrow=.5] (12.center) to (13.center);
 \end{pgfonlayer}
 \end{tikzpicture}
 \hspace{5mm}
 \begin{tikzpicture}
 \begin{pgfonlayer}{nodelayer}
		\node [style=none, label=above:{$A$}] (3) at (-1.25, 1) {};
		\node [style=none] (4) at (-0.25, 1) {};
		\node [style=none, label=below: {$A$}] (5) at (0.75, 1) {};
		\node [style=none] (22) at (1.25, 1) {=};
		\node [style=none, label=above:{$A$}] (16) at (1.75, 1.5) {};
		\node [style=none, label=below:{$A$}] (17) at (1.75, 0.5) {};
		\end{pgfonlayer}
\begin{pgfonlayer}{edgelayer}
		\draw [arrow=.5] [bend right=90, looseness=2.25] (3.center) to (4.center);
		\draw [arrow=.5] [bend left=90, looseness=2.25] (4.center) to (5.center);
		\draw [arrow=.5] (16.center) to (17.center);
 \end{pgfonlayer}
 \end{tikzpicture}
 
%  \hspace{5mm}
 \begin{tikzpicture}
 \begin{pgfonlayer}{nodelayer}
		\node [style=none] (6) at (-0.25, -1) {};
		\node [style=none, label=below:{$A^r$}] (7) at (-1.25, -1) {};
		\node [style=none,label=above:{$A^r$}] (8) at (0.75, -1) {};
		\node [style=none] (23) at (1.25, -1) {=};
		\node [style=none, label=above:{$A^r$}] (18) at (1.75, -0.5) {};
		\node [style=none,  label=below:{$A^r$}] (19) at (1.75, -1.5) {};
		\end{pgfonlayer}
\begin{pgfonlayer}{edgelayer}
		\draw [arrow=.5] [bend left=90, looseness=2.25] (7.center) to (6.center);
		\draw [arrow=.5] [bend right=90, looseness=2.25] (6.center) to (8.center);
		\draw [arrow=.5] (18.center) to (19.center);
 \end{pgfonlayer}
 \end{tikzpicture}
  \hspace{5mm}
 \begin{tikzpicture}
 \begin{pgfonlayer}{nodelayer}
		\node [style=none, label=above:{$A^l$}] (9) at (-6, -1) {};
		\node [style=none, label=below:{$A^l$}] (10) at (-4, -1) {};
		\node [style=none] (11) at (-5, -1) {};
		\node [style=none, label=above:{$A^l$}] (14) at (-3, -0.5) {};
		\node [style=none,  label=below:{$A^l$}] (15) at (-3, -1.5) {};
		\node [style=none] (21) at (-3.5, -1) {=};
		\end{pgfonlayer}
	\begin{pgfonlayer}{edgelayer}
		\draw [arrow=.5] [bend left=90, looseness=2.25] (11.center) to (9.center);
		\draw [arrow=.5] [bend right=90, looseness=2.25] (10.center) to (11.center);
		\draw [arrow=.5] (14.center) to (15.center);
	\end{pgfonlayer}
\end{tikzpicture}
\end{center}

More generally, the reduction rules for diagrammatic calculus allow continuous deformations.  
One such deformation we will make use of is the \textit{swing rule}:

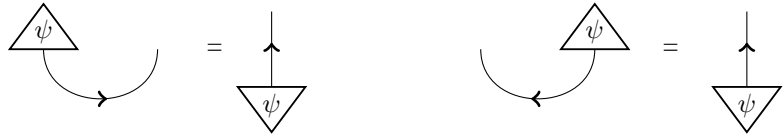
\begin{center}
\begin{tikzpicture}
	\begin{pgfonlayer}{nodelayer}
		\node [state, hflip] (0) at (-4, 0.5) {$\psi$};
		\node [style=none] (1) at (-2.5, 0.5) {};
		\node [style=none] (2) at (-1.75, 0.5) {=};
		\node [state] (3) at (-1, 0) {$\psi$};
		\node [style=none] (4) at (-1, 1) {};
		\node [style=none] (5) at (5.25, 1) {};
		\node [style=none] (6) at (1.75, 0.5) {};
		\node [state, hflip] (7) at (3.25, 0.5) {$\psi$};
		\node [state] (8) at (5.25, 0) {$\psi$};
		\node [style=none] (9) at (4.25, 0.5) {=};
	\end{pgfonlayer}
	\begin{pgfonlayer}{edgelayer}
		\draw [arrow=.5] (3) to (4.center);
		\draw [arrow=.5, bend left=90, looseness=1.50] (7) to (6.center);
		\draw [arrow=.5](8) to (5.center);
		\draw [arrow=.5, bend right=90, looseness=1.50] (0) to (1.center);
	\end{pgfonlayer}
\end{tikzpicture}
\end{center}

\begin{definition} \citep{lambek2001}
 A \textbf{pregroup} $(P, \leq, \cdot, 1, (-)^l, (-)^r )$ is a partially ordered monoid in which each element $a$ has both a left adjoint $a^l$ and 
 a right adjoint $a^r$ such that 
 $a^l a \leq 1 \leq aa^l \text{ and } aa^r \leq 1 \leq a^r a$.
 \end{definition}
If $a \leq b$ it is common practice to write $a \ra b$ and say that $a$ reduces to $b$. This terminology is useful when pregroups 
are applied to natural language, where each word gets assigned a pregroup type freely generated from a set
of basic elements. The sentence is deemed to be grammatical if the concatenation of the types of the words 
reduce to the simple type of a sentence. For example
reduction for a simple transitive sentence is  $n (n^r s n^l) n \ra 1 s n^l n \ra 1 s 1 \ra s$.

%
% The basic types that will used in this paper are the following:
% \begin{align*}
% &n: \text{ noun } &s: \text{ declarative statement } \\
% &j: \text{ infinitive of the verb } &\sigma: \text{ glueing type }
% \end{align*}
 A pregroup \cat{P} is a concrete instance of a compact closed category. The $\eta^l, \eta^r, \epsilon^l,
\epsilon^r$ maps are 
$\eta^l = [ 1 \leq p \cdot p^l ]$, $\epsilon^l =  [ p^l \cdot p \leq 1 ]$, 
$\eta^r = [ 1 \leq p^r \cdot p ]$,  $\epsilon^r = [  p \cdot p^r \leq 1 ]$.
\paragraph{FVect as a concrete compact closed category.} 
Finite dimensional vector spaces over the base field $\R$, together with linear maps form a monoidal category,
referred to as \cat{FVect}.
The monoidal tensor is the usual vector space tensor and the monoidal unit is the base field
$\R$. It is also a compact closed category where
$V^l = V^r =  V$.
The compact closed maps are defined as follows: 

Given a vector space $V$ with basis $\{ \vec{e_i} \}_i$,
\begin{align*}
\begin{aligned}
\eta^l_V = \eta^r_V : \,\, &\R \ra V \tns V \\
&1 \mapsto \sum_i \vec{e_i} \tns \vec{e_i} 
\end{aligned}
&&
\begin{aligned}
\epsilon^l_V = \epsilon^r_V : \,\, &V \tns V \ra \R \\
 &\sum_{ij} c_{ij} \,\, \vec{v_i} \tns \vec{w_i} \mapsto \sum_{ij} c_{ij} \braket{\vec{v_i}}{\vec{w_i}} 
\end{aligned}
\end{align*}
\paragraph{Categorical representation of meaning space.}
The tensor in \cat{FVect} is commutative up to isomorphism. 
This causes the left and the right adjoints to be the same, and thus for the left 
and the right compact closed maps to coincide. Thus \cat{FVect} by itself cannot take the 
effect of word ordering on meaning into account. \cite{coecke2010}  propose a way around this 
obstacle by considering the product category $\cat{FVect} \times \cat{P}$ where $\cat{P}$ is a pregroup. 

Objects in \cat{FVect \times P} are of the form $ (V, p)$, where $V$ is the vector space for the representation of meaning and $p$ is the 
pregroup type. There exists a morphism $(f, \leq): (V,p) \ra (W, q)$ if there exists a morphism $f: V \ra W$ in \cat{FVect}
and $p \leq q$ in \cat{P}. 

The compact closed structure of \cat{FVect} and \cat{P} lifts componentwise to the product category \cat{FVect \times P}:
\begin{align*}
\eta^l: (\R, 1) \ra ( V \otimes V, p \cdot p^l) &&  
\eta^r: (\R, 1) \ra (V \otimes V, p^r \cdot p) \\
\epsilon^l: (V \otimes V, p^l \cdot p) \ra (\R,1) &&
\epsilon^r: (V \otimes V, p \cdot p^r) \ra (\R, 1) 
\end{align*}
\begin{definition}
An object $(V,p)$ in the product category is called a \textbf{meaning space}, where $V$ is the vector space in which the meanings $\vec{v} \in V$
of strings of type $p$ live. 
\end{definition}

\begin{definition}\label{sentenceMap} \textbf{From-meanings-of-words-to-the-meaning-of-the-sentence map.}
Let $v_1 v_2 \ldots v_n$ be a string of words, each $v_i$ with a meaning space representation $\vec{v_i} \in (V_i, p_i)$. 
Let $x \in P$ be a pregroup type such that $[ p_1 p_2 \ldots p_n \leq x ]$ Then the meaning vector for the string is $\vec{ v_1 v_2 \ldots v_n } := f(\vec{v_1} \otimes \vec{v_2} \otimes \ldots \otimes \vec{v_n})  \in (W, x) $,
where $f$ is defined to be the application of the compact closed maps obtained from the reduction $[ p_1 p_2 \ldots p_n \leq x ]$
to the composite vector space $V_1 \otimes V_2 \otimes \ldots \otimes V_n$.
\end{definition}

This framework uses the maps of the pregroup reductions and the elements of objects in \cat{FVect}. 
The diagrammatic calculus provides a tool to reason about both. As an example, take the sentence ``John likes Mary''. It has the pregroup type $n n^r s n^l n$,
 and the vector representations
$\vec{John}, \vec{Mary} \in V$ and $\vec{likes} \in V \otimes S \otimes V$. 
The morphism in \cat{FVect \times P} corresponding to 
the map defined in Definition \ref{sentenceMap}  is of type $( V \otimes (V \otimes S \otimes V) \otimes V, n n^r s n^l n ) \ra (S, s)$.
From the pregroup reduction $[ n n^r s n^l n \ra s ]$ we obtain the compact closed maps 
$\epsilon^r 1 \epsilon^l$. In \cat{FVect} 
this translates into $\epsilon_V \otimes 1_S \otimes \epsilon_V:   V \otimes (V \otimes S \otimes V) \otimes V \ra S $.
 This map, when applied to $\vec{John} \otimes \vec{likes} \otimes \vec{Mary}$ has the following depiction in the diagrammatic calculus:

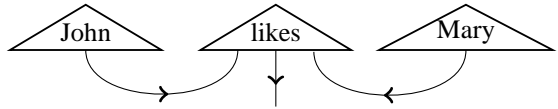
\begin{center}
\begin{tikzpicture}
	\begin{pgfonlayer}{nodelayer}
		\node [style=state, hflip] (0) at (0, 0) {likes};
		\node [style=none] (1) at (-0.5, 0) {};
		\node [style=none] (2) at (0.5, 0) {};
		\node [style=state,hflip] (3) at (-2.5, 0) {John};
		\node [style=state,hflip] (4) at (2.5, 0) {Mary};
		\node [style=none] (5) at (0, -0.75) {};
	\end{pgfonlayer}
	\begin{pgfonlayer}{edgelayer}
		\draw [arrow=.5] [bend right=90, looseness=1] (3) to (1.center);
		\draw[arrow=.5] [bend left=90, looseness=1] (4.center) to (2);
		\draw [arrow=.5] (0) to (5.center);
	\end{pgfonlayer}
\end{tikzpicture}
\end{center}

Note that this construction treats the verb `likes' essentially as a 
relation that takes two inputs of type $V$, and outputs a vector 
of type $S$. 
For the explicit calculation, note that $ \vec{likes} = \sum_{ijk} c_{ijk} \vec{v_i} \otimes \vec{s_j} \otimes \vec{v_k}$, where $\{\vec{v_i}\}_i$ is an orthonormal basis for $V$ and $\{\vec{s_j}\}_j$ is an orthonormal basis for $S$. Then
\begin{align}
 \vec{John \,\, likes \,\,Mary}  &= \epsilon_V \otimes 1_S \otimes \epsilon_V (\vec{John} \otimes \vec{likes} \otimes\vec{Mary}) \\
&= \sum_{ijk} \braket{\vec{John}}{\vec{v_i}} \vec{s_j} \braket{\vec{v_k}}{\vec{Mary}} 
\end{align} 

The reductions in diagrammatic calculus help reduce the final calculation to a simpler term. 
The non-reduced reduction,
when expressed in dirac notation is $( \bra{\epsilon^r_V } \otimes 1_S \otimes \bra{\epsilon^l_V} ) \circ \ket{\vec{ John} \otimes \vec{likes} \otimes \vec{Mary} }$.
But we can \emph{swing} $\vec{John}$ and $\vec{Mary}$ 	in accord with the reduction rules 
in the diagrammatic calculus. The diagram 
then reduces to:

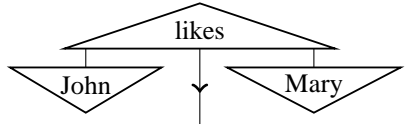
\begin{center}
\begin{tikzpicture}
	\begin{pgfonlayer}{nodelayer}
		\node [style=state,hflip, xscale=4] (0) at (0, 0) {};
		\node [style=none] (6) at (0,0.25) {likes};
		\node [style=none] (1) at (-1.50, 0) {};
		\node [style=none] (2) at (1.50, 0) {};
		\node [style=none] (3) at (0, -1) {};
		\node [style=state] (4) at (-1.50, -0.25) {John};
		\node [style=state] (5) at (1.50, -0.25) {Mary};
	\end{pgfonlayer}
	\begin{pgfonlayer}{edgelayer}
		\draw [arrow=.5] (0) to (3.center);
		\draw (1.center) to (4);
		\draw (2.center) to (5);
	\end{pgfonlayer}
\end{tikzpicture}
\end{center}

This results in a simpler expression that needs to be calculated: $(\bra{\vec{John}} \otimes 1_S \otimes \bra{\vec{Mary}})  \circ \ket{\vec{likes}}$.

\section{Density Matrices as Elements of a Compact Closed Category}\label{cpm}

Recall that in \cat{FVect}, vectors $\ket{\vec{v}} \in V$
are in one-to-one correspondence with morphisms of type $v: I \ra V$. Likewise, pure states of the form $\ketbra{\vec{v}}{\vec{v}}$ 
are in one-to-one correspondence with morphisms $v \circ v^\dagger: V \ra V$ such that $v^\dagger \circ v = \text{id}_I$, where $v^\dagger$ denotes the adjoint of $v$
(notice that this corresponds to the condition that $ \braket{v}{v} = 1$).  A general (mixed) state $\rho$ is a positive morphism
of the form $\rho: V \ra V$. 
One can re-express the mixed states $\rho: V \ra V$ as elements $\rho: I \ra V^* \otimes V$. 
Here $V^*= V^l = V^r  = V$. 
\begin{definition}
 $f$ is a \textbf{completely positive map} if $f$ is positive for any positive operator $A$, and
 $(\text{ id}_V \otimes f) B$ is positive for any positive operator $B$ and any space $V$.  
\end{definition}

Completely positive maps in \cat{FVect} form a monoidal category \citep{selinger2007}.
Thus one can define a new category \cat{CPM(FVect)} where the objects of \cat{CPM(FVect)} are the same as those of \cat{FVect}, and
morphisms $A \ra B$ in \cat{CPM(FVect)} are completely positive maps $A^* \tns A \ra B^* \tns B$
in \cat{FVect}. The elements $I \ra A$ in \cat{CPM(FVect)} are of the form 
$I^* \tns I \ra A^* \tns A$ in \cat{FVect}, providing a monoidal category with density 
matrices as its elements. 

\paragraph{CPM(FVect) in graphical calculus.}
 A morphism $\rho: A \ra A$ is positive if and only if there exists a map $\sqrt{\rho}$ 
 such that $\rho = \sqrt{\rho}^\dagger \circ \sqrt{\rho}$. In \cat{FVect}, the isomorphism between
 $\rho: A \ra A$ and $\ulcorner\rho\urcorner : I \ra A^* \tns A$ is provided by $\eta^l = \eta^r$. 
The graphical representation of $\rho$ in \cat{FVect} then becomes:
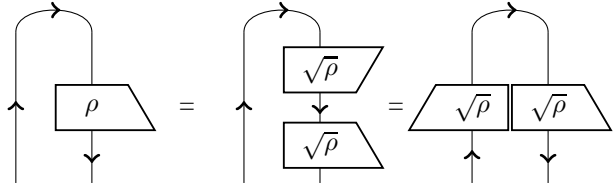
\tikzset{wedge}
\begin{center}
\begin{tikzpicture}
	\begin{pgfonlayer}{nodelayer}
		\node [style=none] (0) at (-5, 1) {};
		\node [style=none] (1) at (-4, 1) {};
		\node [morphism] (2) at (-4, 0) {$\rho$};
		\node [style=none] (3) at (-5, -1) {};
		\node [style=none] (4) at (-4, -1) {};
		\node [style=none] (5) at (-2.75, 0) {=};
		\node [style=none] (6) at (-2, -1) {};
		\node [style=none] (7) at (-2, 1) {};
		\node [style=none] (8) at (-1, 1) {};
		\node [style=none] (9) at (-1, -1) {};
		\node [morphism, hflip] (10) at (-1, 0.5) {$\sqrt{\rho}$};
		\node [morphism] (11) at (-1, -0.5) {$\sqrt{\rho}$};
		\node [style=none] (12) at (0, 0) {=};
		\node [style=none] (13) at (1, -1) {};
		\node [style=none] (14) at (1, 1) {};
		\node [style=none] (15) at (2, 1) {};
		\node [style=none] (16) at (2, -1) {};
		\node [morphism, vflip] (17) at (1, 0) {$\sqrt{\rho}$};
		\node [morphism] (18) at (2, 0) {$\sqrt{\rho}$};
	\end{pgfonlayer}
	\begin{pgfonlayer}{edgelayer}
		\draw [arrow=0.5, bend left=90, looseness=1] (0.center) to (1.center);
		\draw  (1.center) to (2.north);
		\draw [arrow=0.5] (2.south) to (4.center);
		\draw [arrow=0.5] (3.center) to (0.center);
		\draw [arrow=0.5] (6.center) to (7.center);
		\draw [arrow=0.5] [bend left=90, looseness=1] (7.center) to (8.center);
		\draw  (8.center) to (10.north);
		\draw  [arrow=0.5] (10.south) to (11.north);
		\draw  (11.south) to (9.center);
		\draw [arrow=0.5] (13.center) to (17.south);
		\draw  (17.north) to (14.center);
		\draw [arrow=0.5] [bend left=90, looseness=1] (14.center) to (15.center);
		\draw  (15.center) to (18.north);
		\draw [arrow=0.5] (18.south) to (16.center);
	\end{pgfonlayer}
\end{tikzpicture} 
\end{center}
Notice that the categorical definition of a positive morphism coincides with the definition of a positive operator in a vector space, where $\sqrt{\rho}$ is the square root of the operator. 

The graphical depiction of completely positive morphisms come from the following theorem:
 \begin{theorem} (Stinespring Dilation Theorem) 
 $f: A^* \otimes A \ra B^* \otimes B$ is completely positive if and only if there is an object $C$ and a morphism $\sqrt{f}: A \ra C \ts B$ such that the following equation holds: 
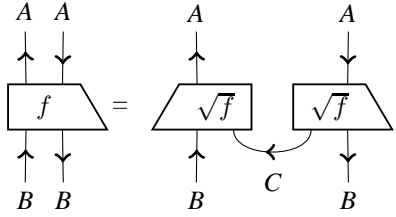
\begin{center}
\begin{tikzpicture}
	\begin{pgfonlayer}{nodelayer}
		\node [morphism] (0) at (-3.75, 0) {$f$};
		\node [morphism] (1) at (0, 0) {$\sqrt{f}$};
		\node [style=none] (2) at (-2.75, 0) {=};
		\node [morphism, vflip] (3) at (-1.5, 0) {$\sqrt{f}$};
		\node [style=none, label=A] (4) at (-4, 1) {};
		\node [style=none, label=A] (5) at (-3.5, 1) {};
		\node [style=none] (6) at (-4, -1) {};
		\node [style=none] (7) at (-3.5, -1) {};
		\node [style=none] (8) at (-1.75, 1) {};
		\node [style=none] (9) at (0.25, 1) {};
		\node [style=none] (10) at (-1.75, -1) {};
		\node [style=none] (11) at (0.25, -1) {};
		\node [style=none] (12) at (-0.75, -1) {C};
	\end{pgfonlayer}
	\begin{pgfonlayer}{edgelayer}
		\draw [arrow=0.5] (0.north west) to (4.center) ;
		\draw [arrow=0.5] (6.center) node[below] {B} to (0.south west);
		\draw [arrow=0.5] (5.center) to (0.north east);
		\draw [arrow=0.5] (0.south east)to (7.center) node[below] {B} ;
		\draw [arrow=0.5] (3.north west) to (8.center) node [above] {A}  ;
		\draw [arrow=0.5] (10.center) node[below]{B} to (3.south west);
		\draw [arrow=0.5] (9.center) node [above] {A} to (1.north east)   ;
		\draw [arrow=0.5] (1.south east) to (11.center) node[below] {B} ;
		\draw [arrow=0.5]  [bend left=90, looseness=1] (1.south west) to (3.south east)  ;
	\end{pgfonlayer}
\end{tikzpicture}
\end{center}
 \end{theorem}
 
$\sqrt{f}$ and $C$ here are not unique. For the proof of the theorem see \cite{selinger2007}.

\begin{theorem}
\cat{CPM(FVect)} is a compact closed category where
as in \cat{FVect}, $V^r = V^l = V$ and the compact closed maps are defined to be:
\begin{align*}
\eta^l = (\eta^r_{V} \tns \eta^l_{V} ) \circ ( \bo{1}_V \tns \sigma \tns \bo{1}_V)  &&
\eta^r = (\eta^l_{V} \tns \eta^r_{V} ) \circ ( \bo{1}_V \tns \sigma \tns \bo{1}_V) \\
\epsilon^l = ( \bo{1}_V \tns \sigma \tns \bo{1}_V) \circ (\epsilon^r_{V} \tns \epsilon^l_{V}) &&
 \epsilon^r = ( \bo{1}_V \tns \sigma \tns \bo{1}_V) \circ (\epsilon^l_{V} \tns \epsilon^r_{V})
 \end{align*}
where $\sigma$ is the swap map defined as $\sigma( v \ts w) = (w \ts v)$.   
\end{theorem}

\begin{proof}
The graphical construction of the compact closed maps boils down to doubling the objects and the wires. The identities are proved by adding bends in the wires. 
Consider the diagram for $\eta^r$:

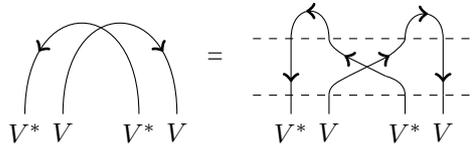
\begin{center}
\begin{tikzpicture}
	\begin{pgfonlayer}{nodelayer}
		\node [style=none] (0) at (1, 0.5) {};
		\node [style=none] (1) at (1.5, 0.5) {};
		\node [style=none] (2) at (1, 0.25) {};
		\node [style=none] (3) at (1.5, 0.25) {};
		\node [style=none] (4) at (0.75, 0.75) {};
		\node [style=none] (5) at (0.25, 0.75) {};
		\node [style=none] (6) at (0.75, 0) {};
		\node [style=none] (7) at (0.25, 0) {};
		\node [style=none] (8) at (1.75, 0.75) {};
		\node [style=none] (9) at (2.25, 0.75) {};
		\node [style=none] (10) at (1.75, 0) {};
		\node [style=none] (11) at (2.25, 0) {};
		\node [style=none] (12) at (0.25, -0.25) {};
		\node [style=none] (13) at (0.75, -0.25) {};
		\node [style=none] (14) at (1.75, -0.25) {};
		\node [style=none] (15) at (2.25, -0.25) {};
		\node [style=none] (16) at (-3.25, -0.25) {};
		\node [style=none] (17) at (-2.75, -0.25) {};
		\node [style=none] (18) at (-1.75, -0.25) {};
		\node [style=none] (19) at (-1.25, -0.25) {};
		\node [style=none] (20) at (-0.75, 0.5) {=};
		\node [style=none] (21) at (-0.25, 0.75) {};
		\node [style=none] (22) at (2.75, 0.75) {};
		\node [style=none] (23) at (-0.25, 0) {};
		\node [style=none] (24) at (2.75, 0) {};
		\node [style=none] (25) at (-3.25, -0.5) {$V^*$};
		\node [style=none] (26) at (-2.75, -0.5) {$V$};
		\node [style=none] (27) at (-1.75, -0.5) {$V^*$};
		\node [style=none] (28) at (-1.25, -0.5) {$V$};
		\node [style=none] (29) at (0.25, -0.5) {$V^*$};
		\node [style=none] (30) at (0.75, -0.5) {$V$};
		\node [style=none] (31) at (1.75, -0.5) {$V^*$};
		\node [style=none] (32) at (2.25, -0.5) {$V$};
	\end{pgfonlayer}
	\begin{pgfonlayer}{edgelayer}
		\draw (6.center) to (13.center);
		\draw [in=90, out=-150] (2.center) to (6.center);
		\draw [reverse arrow=0] (1.center) to (2.center);
		\draw [in=27, out=-90, looseness=0.50] (8.center) to (1.center);
		\draw [reverse arrow=0.5, bend right=90, looseness=2.25] (9.center) to (8.center);
		\draw [reverse arrow=.3] (11.center) to (9.center);
		\draw (15.center) to (11.center);
		\draw (12.center) to (7.center);
		\draw [reverse arrow=.3] (7.center) to (5.center);
		\draw [reverse arrow=.5, bend left=90, looseness=2.50] (5.center) to (4.center);
		\draw [in=150, out=-90, looseness=0.75] (4.center) to (0.center);
		\draw [reverse arrow=0] (0.center) to (3.center);
		\draw [in=90, out=-15] (3.center) to (10.center);
		\draw (10.center) to (14.center);
		\draw  [arrow=0.7, bend right=90, looseness=2.75] (18.center) to (16.center);
		\draw [arrow=0.7, bend left=90, looseness=2.75] (17.center) to (19.center);
		\draw [dashed] (21.center) to (22.center);
		\draw [dashed] (23.center) to (24.center);
	\end{pgfonlayer}
\end{tikzpicture}
\end{center}
These maps satisfy the axioms of compact closure since the components do.
\end{proof}

The concrete compact closed maps are as follows:
\begin{align*}
\eta^l = \eta^r &: \R \ra (V \tns V) \tns (V \tns V) \\
::& 1 \mapsto \sum_i \vec{e_i} \tns \vec{e_i}  \tns \sum_j \vec{e_j} \tns \vec{e_j} \\
\epsilon^l = \epsilon^r &: (V \tns V) \tns (V \tns V) \ra \R \\
::& \sum_{ijkl} c_{ijkl} \,\, \vec{v_i} \tns \vec{w_j}  \tns \vec{u_k} \tns \vec{p_l} \mapsto
 \sum_{ijkl} c_{ijkl} \,\, \braket{\vec{v_i}} { \vec{u_k}} \braket{ \vec{w_j}} { \vec{p_l}}
\end{align*}
  Let $ \rho : V_1 \tns V_2 \tns \ldots \tns V_n \ra V_1 \tns V_2 \tns \ldots \tns V_n $ be a density operator defined on an arbitrary composite space $V_1 \tns V_2 \tns \ldots \tns V_n$. 
 Then it has the density matrix representation $\rho : I \ra ( V_1 \tns V_2 \tns \ldots \tns V_n )^* \tns ( V_1 \tns V_2 \tns \ldots \tns V_n ) $.
  Since the underlying category \cat{FVect} is symmetric, it has the swap map $\sigma$. This provides us with 
 the isomorphism: 
 \[ ( V_1 \tns V_2 \tns \ldots \tns V_n )^* \tns ( V_1 \tns V_2 \tns \ldots \tns V_n ) \sim
  ( V_1^* \tns V_1 ) \tns (V_2^* \tns V_2) \tns \ldots \tns (V_n^* \tns V_n) \]  
  So $\rho$ can be equivalently expressed as $\rho : I \ra ( V_1^* \tns V_1 ) \tns (V_2^* \tns V_2) \tns \ldots \tns (V_n^* \tns V_n)$. With this addition, we can simplify the diagrams used to express density matrices by using a single 
thick wire for the doubled wires. Doubled compact closed maps can likewise be expressed by a single thick wire.

\begin{center}  
  \begin{tikzpicture}
	\begin{pgfonlayer}{nodelayer}
		\node [style=none] (0) at (-2, 1.75) {};
		\node [style=none] (1) at (-2, 1) {};
		\node [style=none] (2) at (-0.5, 1) {};
		\node [style=none] (3) at (-0.5, 1.75) {};
		\node [style=none] (4) at (-0.25, 1.75) {};
		\node [style=none] (5) at (-0.25, 1) {};
		\node [style=none] (6) at (-1.25, 1.25) {$:=$};
	\end{pgfonlayer}
	\begin{pgfonlayer}{edgelayer}
		\draw[arrow=.5, line width=2] (0.center) to (1.center);
		\draw [arrow=.5](2.center) to (3.center);
		\draw [arrow=.5] (4.center) to (5.center);
	\end{pgfonlayer}
\end{tikzpicture}
\hspace{10mm}
\begin{tikzpicture}
	\begin{pgfonlayer}{nodelayer}
		\node [style=none] (0) at (-1.75, 1) {};
		\node [style=none] (1) at (-0.75, 1) {};
%		\node [style=none] (2) at (-1.75, -0.25) {};
%		\node [style=none] (3) at (-0.75, -0.25) {};
		\node [style=none] (4) at (-1, 1) {};
		\node [style=none] (5) at (-2, 1) {};
%		\node [style=none] (6) at (4.25, -0.25) {};
%		\node [style=none] (7) at (3.25, -0.25) {};
%		\node [style=none] (8) at (-2, -0.25) {};
%		\node [style=none] (9) at (3, -0.25) {};
		\node [style=none] (10) at (4.25, 1) {};
%		\node [style=none] (11) at (-1, -0.25) {};
		\node [style=none] (12) at (4, 1) {};
%		\node [style=none] (13) at (4, -0.25) {};
		\node [style=none] (14) at (3.25, 1) {};
		\node [style=none] (15) at (3, 1) {};
%		\node [style=none] (16) at (1.75, -0.25) {};
		\node [style=none] (17) at (1.75, 1) {};
%		\node [style=none] (18) at (0.75, -0.25) {};
		\node [style=none] (19) at (0.75, 1) {};
		\node [style=none] (20) at (-3.25, 1) {};
%		\node [style=none] (21) at (-4.25, -0.25) {};
		\node [style=none] (22) at (-4.25, 1) {};
%		\node [style=none] (23) at (-3.25, -0.25) {};
		\node [style=none] (24) at (-2.75, 1.25) {$:=$};
%		\node [style=none] (25) at (-2.75, -0.5) {$:=$};
		\node [style=none] (26) at (2.25, 1.25) {$:=$};
%		\node [style=none] (27) at (2.25, -0.5) {$:=$};
	\end{pgfonlayer}
	\begin{pgfonlayer}{edgelayer}
		\draw [arrow=.7, bend left=90, looseness=2.25] (0.center) to (1.center);
%		\draw [arrow=.7,bend right=90, looseness=2.25] (2.center) to (3.center);
		\draw [arrow=.7,bend right=90, looseness=2.25] (4.center) to (5.center);
%		\draw [arrow=.7,bend left=90, looseness=2.25] (6.center) to (7.center);
		\draw [arrow=.7,bend left=90, looseness=2.25] (15.center) to (12.center);
%		\draw [arrow=.7,bend right=90, looseness=2.25] (9.center) to (13.center);
		\draw [arrow=.7,bend right=90, looseness=2.25] (10.center) to (14.center);
%		\draw [arrow=.7,bend left=90, looseness=2.25] (11.center) to (8.center);
		\draw [arrow=.5, line width=2] [bend left=90, looseness=2.25] (22.center) to (20.center);
%		\draw  [arrow=.5,line width=2] [bend right=90, looseness=2.25] (21.center) to (23.center);
		\draw  [arrow=.5, line width=2]  [bend right=90, looseness=2.25] (17.center) to (19.center);
%		\draw  [arrow=.5, line width=2]  [bend left=90, looseness=2.25] (16.center) to (18.center);
	\end{pgfonlayer}
\end{tikzpicture}

\end{center}

The diagrammatic expression of a from-meanings-of-words-to-the-meaning-of-the-sentence map 
using density matrices will therefore look exactly like the depiction of it in \cat{FVect}, but with thick wires. 
\section{Using Density Matrices to Model Meaning}

If one wants to use the full power of density matrices in modelling meaning, one needs to 
establish an interpretation for the distinction 
between \textit{mixing} and \textit{superposition} in the context of linguistics. Let \textit{contextual features} be the salient, quantifiable features of the 
contexts a word is observed in. Let the basis of the space be indexed by such contextual features. Individual \textit{contexts}, such as words in an $n$-word window of a text, can be represented as the superposition of the bases corresponding to the contextual features observed in it. So each context corresponds to a pure state. Words are then probability distributions over the contexts they appear in. The simple co-occurrence model can be cast as a special case of this more general approach, 
where features and contexts are the same. Then all word meanings are mixtures of basis vectors, and they all commute with each other. 

\paragraph{Similarity for density matrices.}
 Fidelity is a good measure of similarity between two density matrix representations of meaning because of its properties listed below.

\begin{definition}
The \textit{fidelity} of two density operators $\rho$ and $\sigma$ is 
$ F( \rho, \sigma) := \tr \sqrt{\rho^{1/2} \sigma \rho^{1/2}}$.
\end{definition}
Some useful properties of fidelity are:
\begin{enumerate}
\item $F(\rho, \sigma) = F(\sigma, \rho)$
\item $0 \leq F(\rho,\sigma) \leq 1$
\item $F(\rho, \sigma) = 1$ if and only if $\rho = \sigma$
\item If $\ketbra{\phi}{\phi}$ and $\ketbra{\psi}{\psi}$ are two pure states, their fidelity is 
equal to $|\braket{\phi}{\psi}|$.
\end{enumerate}

These properties ensure that if the representations of two words are not equal to each other, they will not be judged perfectly similar, and, if two words are represented as projections onto one dimensional subspaces, their similarity value will
be equal to the usual cosine similarity of the vectors. 

\paragraph{Entailment for density matrices.} To develop a theory of entailment using density matrices as the basic representations, we assume the following hypothesis:

\begin{definition}[Distributional Hypothesis for Hyponymy]
 The meaning of a word $w$ subsumes the meaning of a word $v$ if and only 
if it is appropriate to use $w$ in all the contexts $v$ is used.
\end{definition}

This is a slightly more general version of the \textit{Distributional Inclusion Hypothesis} (DIH) stated in \cite{kotlerman2010}.
The difference lies in the additional power the density matrix formalism provides: the distinction between 
mixing and superposition. Further, DIH only considers  whether or not the target word occurs together
with the salient distributional feature at all, and ignores any possible statistically significant correlations of features; here again, the density matrix formalism offers a solution. 

Note that \cite{geffet2005} show that while there is ample evidence for the distributional inclusion hypothesis, this in 
itself does not necessarily provide a method to detect hyponymy-hypernymy pairs.
One of their suggestions for improvement
is to consider more than one word in the features, equivalent to what we do here by taking correlations into account in a co-occurrence 
space where the bases are context words.

 Relative entropy quantifies the distinguishability
of one distribution from another. 
The idea of using relative entropy to model hyponymy is based on the assumption that the distinguishability
of one word from another given its usual contexts provides us with a good metric for hyponymy. For example, if one is given a sentence with the word \textit{dog} crossed out, it will be not be possible for sure to know whether the crossed out word is not \textit{animal} just from the context (except perhaps very particular decelerational sentences which rely on world knowledge, such as `All -- bark'.)

\begin{definition}
The \textbf{(quantum) relative entropy} of two density matrices $\rho$ and $\sigma$ is
$N(\rho || \sigma) := \tr (\rho \log \rho) - \tr (\rho \log \sigma) $,
where $0 \log 0 = 0$ and $x \log 0 = \infty$ when $x \neq 0$ by convention.
\end{definition} 

\begin{definition}
The \textbf{representativeness} between $\rho$ and $\sigma$ is $R(\rho, \sigma) := 1/(1 + N( \rho || \sigma)) $,
where $N( \rho || \sigma)$ is the quantum relative entropy between $\rho$ and $\sigma$.
\end{definition}

Quantum relative entropy is always non-negative.  For two density matrices $\rho$ and $\sigma$, $N( \rho || \sigma) = \infty$ if $\text{supp}(\rho) \cap \ker(\sigma) \neq 0$,
and is finite otherwise. The following is a direct consequence of these properties:

\begin{corollary}\label{kernel}
For all density matrices $\rho$ and $\sigma$, $ R(\rho, \sigma) \leq 1$ with equality if and only if $\rho = \sigma$, 
and $0 \leq R(\rho, \sigma) $ with equality if and only if  
 $\text{supp}(\rho) \cap \text{ker}(\sigma) \neq 0$ 
\end{corollary}

The second part of the corollary reflects the idea that if there is a context in which it is appropriate to use $v$ but not $w$, then
$v$ is perfectly distinguishable from $w$. Such contexts are exactly those that fall within $\text{supp}(\rho) \cap \text{ker}(\sigma)$.

\paragraph{Characterizing hyponyms.} The quantitative measure on density matrices given by  representativeness
provide a qualitative preorder on meaning representations as follows: 
\begin{align*}
\rho \prec \sigma &\text{ if } R(\rho, \sigma) > 0 \\
 \rho \sim \sigma &\text{ if }   \prec \sigma \text{ and } \sigma \prec \rho \\
% \rho \prec \sigma &\text{ if } R(\rho, \sigma) > 0 \text{ and } R(\sigma, \rho) = 0
\end{align*}

\vspace{-8mm}

\begin{proposition}\label{repProp}
The following are equivalent:
\begin{enumerate}
\item $\rho  \prec \sigma$
\item $\text{supp}(\rho) \subseteq \text{supp}(\sigma)$
\item There exists a positive operator $\rho'$ and $p > 0$ such that $\sigma = p \rho +  \rho'$
\end{enumerate}
\end{proposition}
\begin{proof}
$(1) \Rightarrow (2)$ and $(2) \Rightarrow (1)$ follow directly from Corollary \ref{kernel}. 

$(2) \Rightarrow (3)$ since $\text{supp}(\rho) \subseteq \text{supp}(\sigma)$ implies that there exists a $p >0$ 
such that $\sigma - p \rho$ is positive. Setting $\rho' = \sigma - p \rho$ gives the desired equality.

$(3) \Rightarrow (2)$ since $p >0$, and so $ \text{supp}(\rho) \subseteq \text{supp}(\sigma) = \text{supp}(p \rho +  \rho')$.
\end{proof}

The equivalence relation $\sim$ groups any two density matrices $\rho$ and $\sigma$ with $\text{supp}(\rho) = \text{supp}(\sigma)$
into the same equivalence class, thus maps the set of density matrices on a Hilbert space $\mathcal{H}$ onto the set of
projections on $\mathcal{H}$. The projections are in one-to-one correspondence with the subspaces of $\mathcal{H}$ and they form an orthomodular lattice, providing a link to the logical structure of the Hilbert space \cite{van2004} aims to exploit
by using density matrices in IR.

Let $\denst{w}$ and $\denst{v}$ be density matrix representations of the words $v$ and $w$. Then $v$ is 
a hyponym of $w$ in this model if $\denst{v} \prec \denst{w}$ and $\denst{v} \nsim \denst{w}$.

Notice that even though this ordering on density matrices extracts a $yes/no$ answer for the question 
``is $v$ a hyponym of $w$?'', the existence of the quantitative measure lets us to also quantify the extent to which $v$ is a hyponym of $w$. This provides some flexibility in
characterizing hyponymy through density matrices in practice. Instead of calling $v$ a hyponym of $w$ even when 
$R(\denst{v}, \denst{w})$ gets arbitrarily small, one can require the representativeness to be above a certain 
threshold $\epsilon$. This modification, however, has the down side of causing the transivity of hyponymy to fail. 

\section{From meanings of words to the meanings of sentences passage}

As in the case for $ \cat{FVect \times P}$,  $\cat{CPM(FVect) \times P}$ is a compact closed category, 
 where the compact closed maps of \cat{CPM(FVect)} and \cat{P} lift 
component-wise to the product category.

\begin{definition}
A \textbf{meaning space} in this new category is a pair $(V^*  \tns V , p)$ where $V^* \tns V$ is the space 
in which density matrices $v: I \ra V^* \tns V$ of the pregroup type $p$ live. 
\end{definition}

\begin{definition}\label{densitySentence}
Let $v_1 v_2 \ldots v_n$ be a string of words, each $v_i$ with a meaning space representation 
$\denst{v_i} \in (V_i^* \tns V_i, p_i)$. 
Let $x \in P$ be a pregroup type such that $[ p_1 p_2 \ldots p_n \leq x ]$.
 Then the meaning density matrix for the string is defined as:
\[ \denst{ v_1 v_2 \ldots v_n } := f(\denst{v_1} \otimes \denst{v_2} \otimes \ldots \otimes \denst{v_n})  \in (W^* \tns W, x) \] 
where $f$ is defined to be the application of the compact closed maps obtained from the
 reduction $[ p_1 p_2 \ldots p_n \leq x ]$
to the composite density matrix space 
$(V_1 \tns V_1^*) \otimes (V_2^* \tns V_2) \otimes \ldots \otimes (V_n^* \tns V_n)$.
\end{definition}

From a high level perspective, the reduction diagrams for  $\cat{CPM(FVect) \times P}$ look no different than 
the original diagrams for $ \cat{FVect \times P}$, except that we depict them with thick instead of thin wires. 
Consider the previous example: ``John likes Mary''. It has the pregroup type
$n (n^r s n^l) n$, and the compact closed maps obtained from the pregroup reduction is 
 $(\epsilon^r \tns 1 \tns \epsilon^l)$.

One can also depict the diagram together with the internal anatomy of the density representations in \cat{FVect}:
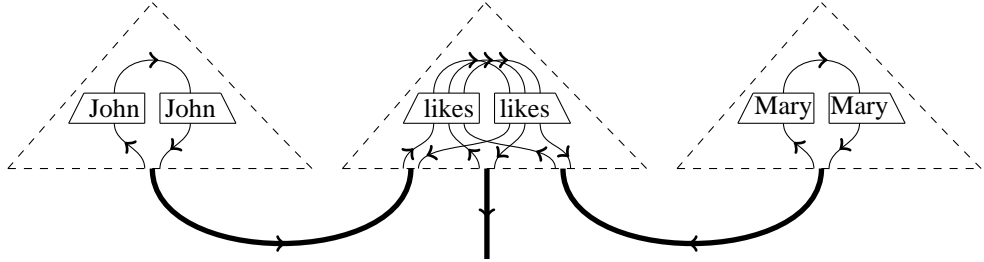
\begin{center}
\begin{tikzpicture}[thick, scale=0.8]
	\begin{pgfonlayer}{nodelayer}
		\node [style=none] (0) at (-1, 0.75) {};
		\node [style=none] (1) at (0, 0.75) {};
		\node [style=none] (2) at (-1.25, 0.25) {};
		\node [style=none] (3) at (0, 0.25) {};
		\node [style=none] (4) at (0.25, 0.75) {};
		\node [style=none] (5) at (1.25, 0.75) {};
		\node [style=none] (6) at (0.25, 0.25) {};
		\node [style=none] (7) at (1.5, 0.25) {};
		\node [style=none] (8) at (-0.75, 0.75) {};
		\node [style=none] (9) at (0.5, 0.75) {};
		\node [style=none] (10) at (-0.5, 0.75) {};
		\node [style=none] (11) at (0.75, 0.75) {};
		\node [style=none] (12) at (-0.25, 0.75) {};
		\node [style=none] (13) at (1, 0.75) {};
		\node [style=none] (14) at (-1.25, -0.5) {};
		\node [style=none] (15) at (-0.75, 0.25) {};
		\node [style=none] (16) at (0.5, 0.25) {};
		\node [style=none] (17) at (-1, -0.5) {};
		\node [style=none] (18) at (1, 0.25) {};
		\node [style=none] (19) at (1.5, -0.5) {};
		\node [style=none] (20) at (1.25, -0.5) {};
		\node [style=none] (21) at (-0.25, 0.25) {};
		\node [style=none] (22) at (0, -0.5) {};
		\node [style=none] (23) at (-0.5, 0.25) {};
		\node [style=none] (24) at (0.75, 0.25) {};
		\node [style=none] (25) at (0.25, -0.5) {};
		\node [style=none] (26) at (-2.25, -0.5) {};
		\node [style=none] (27) at (0.1, 2.25) {};
		\node [style=none] (28) at (2.75, -0.5) {};
		\node [style=none] (29) at (-1.125, -0.5) {};
		\node [style=none] (30) at (0.125, -0.5) {};
		\node [style=none] (31) at (1.375, -0.5) {};
		\node [style=none] (32) at (0.125, -2) {};
		\node [style=none] (33) at (-0.5, 0.5) {likes};
		\node [style=none] (34) at (0.75, 0.5) {likes};
		\node [style=none] (35) at (-4, 0.25) {};
		\node [style=none] (36) at (-6, 0.25) {};
		\node [style=none] (37) at (-5.375, 2.25) {};
		\node [style=none] (38) at (-6, 0.5) {John};
		\node [style=none] (39) at (-5.5, 0.75) {};
		\node [style=none] (40) at (-4.75, 0.25) {};
		\node [style=none] (41) at (-6, 0.75) {};
		\node [style=none] (42) at (-6.5, 0.75) {};
		\node [style=none] (43) at (-4.25, 0.75) {};
		\node [style=none] (44) at (-4.75, 0.5) {John};
		\node [style=none] (45) at (-5.25, -0.5) {};
		\node [style=none] (46) at (-5.5, -0.5) {};
		\node [style=none] (47) at (-5.375, -0.5) {};
		\node [style=none] (48) at (-5.25, 0.75) {};
		\node [style=none] (49) at (-5.25, 0.25) {};
		\node [style=none] (50) at (-7.75, -0.5) {};
		\node [style=none] (51) at (-4.75, 0.75) {};
		\node [style=none] (52) at (-5.5, 0.25) {};
		\node [style=none] (53) at (-2.75, -0.5) {};
		\node [style=none] (54) at (-6.75, 0.25) {};
		\node [style=none] (55) at (5, 0.5) {Mary};
		\node [style=none] (56) at (5.5, 0.75) {};
		\node [style=none] (57) at (7, 0.25) {};
		\node [style=none] (58) at (5.75, -0.5) {};
		\node [style=none] (59) at (5.625, -0.5) {};
		\node [style=none] (60) at (6.75, 0.75) {};
		\node [style=none] (61) at (5, 0.25) {};
		\node [style=none] (62) at (5.5, 0.25) {};
		\node [style=none] (63) at (5.625, 2.25) {};
		\node [style=none] (64) at (4.25, 0.25) {};
		\node [style=none] (65) at (6.25, 0.75) {};
		\node [style=none] (66) at (5.75, 0.25) {};
		\node [style=none] (67) at (6.25, 0.25) {};
		\node [style=none] (68) at (5.5, -0.5) {};
		\node [style=none] (69) at (4.5, 0.75) {};
		\node [style=none] (70) at (5.75, 0.75) {};
		\node [style=none] (71) at (8.25, -0.5) {};
		\node [style=none] (72) at (6.25, 0.5) {Mary};
		\node [style=none] (73) at (5, 0.75) {};
		\node [style=none] (74) at (3.25, -0.5) {};
	\end{pgfonlayer}
	\begin{pgfonlayer}{edgelayer}
		\draw (0.center) to (1.center);
		\draw (1.center) to (3.center);
		\draw (3.center) to (2.center);
		\draw (2.center) to (0.center);
		\draw (4.center) to (5.center);
		\draw (5.center) to (7.center);
		\draw (7.center) to (6.center);
		\draw (6.center) to (4.center);
		\draw [arrow=0.5, bend left=90, looseness=1.50] (8.center) to (9.center);
		\draw [arrow=0.5, bend left=90, looseness=1.50] (10.center) to (11.center);
		\draw [arrow=0.5, bend left=90, looseness=1.50] (12.center) to (13.center);
		\draw [dashed] (26.center) to (27.center);
		\draw [dashed] (28.center) to (27.center);
		\draw [dashed] (26.center) to (28.center);
		\draw [arrow=0.3, in=-90, out=90, looseness=1.25] (14.center) to (15.center);
		\draw [arrow=0.8, in=90, out=-90, looseness=0.75] (16.center) to (17.center);
		\draw [arrow=0.3, in=-90, out=90] (22.center) to (23.center);
		\draw [arrow=0.7, in=90, out=-90] (24.center) to (25.center);
		\draw [arrow=0.2, in=-90, out=90, looseness=0.75] (20.center) to (21.center);
		\draw [arrow=0.7, in=90, out=-90, looseness=1.25] (18.center) to (19.center);
		\draw [line width=2, arrow=0.5] (30.center) to (32.center);
		\draw (42.center) to (39.center);
		\draw (39.center) to (52.center);
		\draw (52.center) to (54.center);
		\draw (54.center) to (42.center);
		\draw (48.center) to (43.center);
		\draw (43.center) to (35.center);
		\draw (35.center) to (49.center);
		\draw (49.center) to (48.center);
		\draw [arrow=0.5, bend left=90, looseness=1.50] (41.center) to (51.center);
		\draw [dashed] (50.center) to (37.center);
		\draw [dashed, in=-45, out=135] (53.center) to (37.center);
		\draw [dashed] (50.center) to (53.center);
		\draw [arrow=0.5, in=-90, out=90] (46.center) to (36.center);
		\draw [arrow=0.5, in=90, out=-90] (40.center) to (45.center);
		\draw (69.center) to (56.center);
		\draw (56.center) to (62.center);
		\draw (62.center) to (64.center);
		\draw (64.center) to (69.center);
		\draw (70.center) to (60.center);
		\draw (60.center) to (57.center);
		\draw (57.center) to (66.center);
		\draw (66.center) to (70.center);
		\draw [arrow=0.5, bend left=90, looseness=1.50] (73.center) to (65.center);
		\draw [dashed] (74.center) to (63.center);
		\draw [dashed, in=-45, out=135] (71.center) to (63.center);
		\draw [dashed] (74.center) to (71.center);
		\draw [arrow=0.5, in=-90, out=90] (68.center) to (61.center);
		\draw [arrow=0.5, in=90, out=-90] (67.center) to (58.center);
		\draw [arrow=0.5, line width=2, bend right=90] (47.center) to (29.center);
		\draw [arrow=0.5, line width=2, bend left=90] (59.center) to (31.center);
	\end{pgfonlayer}
\end{tikzpicture}
\end{center}

The graphical reductions for compact closed categories can be applied to the diagram, establishing $(\epsilon^r \tns 1 \tns \epsilon^l) (\denst{John} \tns \denst{likes} \tns \denst{Mary} )=
(\denst{John} \tns 1 \tns \denst{Mary} ) \circ \denst{likes}$.

As formalised in natural logic, one expects that if the subject and object of a sentence are common nouns which are, together with the verb of the sentence,  moreover, upward monotone, then if these are  replaced by their hyponyms, then the meanings of the original and 
the modified sentences would preserve this hyponymy. The following proposition shows that
 the sentence meaning map for simple transitive sentences achieves exactly that.

\begin{theorem} \label{main}
If $\rho, \sigma, \delta, \gamma \in (N^* \tns N, n)$, $\alpha,\beta \in  (N^* \tns N \tns S^* \tns S \tns N^* \tns N, n^l s n^r$  $\rho \prec \sigma$, $\delta \prec \gamma$ and $\alpha \prec \beta$ then 
\[ f(\rho \tns \alpha \tns \delta) \prec f(\sigma \tns \beta \tns \gamma) \]
\noindent
where $f$ is the from-meanings-of-words-to-the-meaning-of-the-sentence map in definition \ref{densitySentence}.
\end{theorem}
\begin{proof}
If $\rho \prec \sigma$, $\delta \prec \gamma$, and $\alpha \prec \beta$, then there exists a positive operator $\rho'$ and  $r >0$ such that
 $\sigma = r \rho +  \rho'$, a positive operator $\delta'$ and $d>0$ such that $\gamma = d \delta + \delta'$ and a positive operator $\alpha'$ and $a >0$ such that $\beta = a \alpha + \alpha'$ by Proposition \ref{repProp}. Then
\begin{align*}
f(\sigma \tns \beta \tns \gamma) &= (\epsilon^r \tns 1 \tns \epsilon^l) (\sigma \tns \beta \tns \gamma) \\
&= (\sigma \tns 1 \tns \gamma) \circ \beta \\
&= ((r \rho +  \rho') \tns 1 \tns (d \delta + \delta')) \circ (a \alpha + \alpha')  \\
&= ( r \rho \tns 1 \tns d \delta) \circ  (a \alpha + \alpha') +  (\rho'  \tns 1 \tns \delta') \circ (a \alpha + \alpha') \\
&= (r \rho\tns 1 \tns d \delta) \circ a \alpha + (r \rho\tns 1 \tns d \delta) \circ \alpha' +
(\rho'  \tns 1 \tns \delta') \circ (a \alpha + \alpha') , \\
f(\rho \tns \alpha \tns \delta) &=  (\rho \tns 1 \tns \delta) \circ \alpha
\end{align*}
since $r, d, a \neq 0$, $\text{supp}(  f(\rho \tns \alpha \tns \delta) ) 
\subseteq \text{supp} (f(\sigma \tns \beta \tns \gamma))$, which by Proposition \ref{repProp}
proves the theorem.
\end{proof}

\section{Truth Theoretic Examples}
We present several examples that demonstrate the application of the 
from-meanings-of-words-to-the-meaning-of-sentence map, where the initial 
meaning representations of words are density matrices, and explore 
how the hierarchy on nouns induced by their density matrix representations 
carry over to a hierarchy in the sentence space.

\subsection{Entailment between nouns}
Let ``lions'', ``sloths''. ``plants'' and ``meat'' have one dimensional representations in the noun space of our model:
 \[ \denst{lions} = \ketbra{\vec{lions}}{\vec{lions}} \hspace{20mm} \denst{sloths} = \ketbra{\vec{sloths}}{\vec{sloths}}\]
 \[\denst{meat} = \ketbra{\vec{meat}}{\vec{meat}}\hspace{20mm} \denst{plants}  = \ketbra{\vec{plants}}{\vec{plants}}\]

Let the representation of ``mammals'' be a mixture of one dimensional representations of individual animals:
\[ \denst{mammals} = 1/2 \ketbra{\vec{lions}}{\vec{lions}} + 1/2 \ketbra{\vec{sloths}}{\vec{sloths}}\]
Notice that 
\begin{align*}
 N(\denst{lions}|| \denst{mammals}) &=  \tr ( \denst{lions}   \log \denst{lions}) -  tr ( \denst{lions} \log \denst{mammals}) \\
&= \log1 - \frac{1}{2} \log \frac{1}{2} = 1 
\end{align*}
Hence $R(\denst{lions}, \denst{mammals}) = 1/2$. For the other direction, 
%\begin{align*}
%N(\denst{mammals} &|| \denst{lions}) \\
%&= \tr (\denst{mammals} \log \denst{mammals}) - \tr( \denst{mammals \log \denst{lions}) \\
%&= \infty
%\end{align*}
since the intersection of the support of $\denst{mammals}$ and the kernel of $\denst{lions}$ is non-empty,  $R(\denst{mammals}, \denst{lions}) = 0$.
This confirms that $\denst{lions} \prec \denst{mammals}$. 

\subsection{Entailment between sentences in one dimensional truth theoretic space} 
Consider a sentence space that is one dimensional, where 1 stands for true and 0 for false. 
Let  sloths eat plants and lions eat meat; this is represented as follows
\begin{align*}
\denst{eat} =  &(\ket{\vec{sloths}} \ket{\vec{plants}} + \ket{\vec{lions}}\ket{\vec{meat}})(\bra{\vec{sloths}}\bra{\vec{plants}} + \bra{\vec{lions}}\bra{\vec{meat}} ) \\
%= &(\ket{\vec{sloths}} \ket{\vec{plants}} )( \bra{\vec{sloths}}\bra{\vec{plants}}) + \\
%&\hspace{10mm}  (\ket{\vec{sloths}} \ket{\vec{plants}} )( \bra{\vec{lions}}\bra{\vec{meat}}) + \\
%&\hspace{10mm} ( \ket{\vec{lions}}\ket{\vec{meat}} ) ( \bra{\vec{sloths}}\bra{\vec{plants}}  ) +\\
%&\hspace{10mm} ( \ket{\vec{lions}}\ket{\vec{meat}} ) ( \bra{\vec{lions}}\bra{\vec{meat}} )\\
\approx &(\prj{sloths} \tns \prj{plants}) +  (\prjdif{sloths}{lions} \tns \prjdif{plants}{meat} ) +\\
& (\prjdif{lions}{sloths} \tns \prjdif{meat}{plants} )+ (\prj{lions} \tns \prj{meat} )
\end{align*}
The above is the density matrix representation of a pure composite state that relate ``sloths'' to ``plants'' and 
``lions'' to ``meat''. If we fix the bases $\{\vec{lions}, \vec{sloths}\}$ for $N_1$, and $\{\vec{meat}, \vec{plants}\}$
for $N_2$, we will have $\denst{eat} : N_1 \tns N_1 \ra N_2 \tns N_2$ with the following matrix representation:

\[ \left( \begin{array}{cccc}
1 & 0 & 0 & 1\\
0 & 0 & 0 & 0\\
0 & 0 & 0 & 0\\
1 & 0 & 0 & 1 \end{array} \right)\] 

\paragraph{``Lions eat meat''}. This is a transitive sentence, so as before, it has the pregroup type:
$n n^l s n^r n$. 
%The diagrammatic expression of the pregroup reduction is as follows:
%
%\begin{center}
%\begin{tikzpicture}
%	\begin{pgfonlayer}{nodelayer}
%		\node [style=ket] (0) at (0, 0) {$\denst{eat}$};
%		\node [style=none] (1) at (-0.5, 0) {};
%		\node [style=none] (2) at (0.5, 0) {};
%		\node [style=ket] (3) at (-5, 0) {$\denst{lions}$};
%		\node [style=ket] (4) at (5, 0) {$\denst{meat}$};
%		\node [style=none] (5) at (0, -1.5) {};
%	\end{pgfonlayer}
%	\begin{pgfonlayer}{edgelayer}
%		\draw [arrow=.5, line width=2, bend right=90, looseness=1.00] (3) to (1.center);
%		\draw[arrow=.5, line width=2, bend left=90, looseness=1.00] (4.center) to (2);
%		\draw [arrow=.5, line width=2] (0) to (5.center);
%	\end{pgfonlayer}
%\end{tikzpicture}
%\end{center}
%
%This reduces to:
%
%\begin{center}
%\begin{tikzpicture}
%	\begin{pgfonlayer}{nodelayer}
%		\node [style=ket, xscale=4] (0) at (0, 0) {};
%		\node [style=none] (6) at (0,0.25) {$\denst{eat}$};
%		\node [style=none] (1) at (-1.50, 0) {};
%		\node [style=none] (2) at (1.50, 0) {};
%		\node [style=none] (3) at (0, -1) {};
%		\node [style=bra] (4) at (-1.50, -0.50) {lions};
%		\node [style=bra] (5) at (1.50, -0.50) {meat};
%	\end{pgfonlayer}
%	\begin{pgfonlayer}{edgelayer}
%		\draw [line width=2, arrow=.8] (0) to (3.center);
%		\draw [line width=2, reverse arrow=.5] (1.center) to (4);
%		\draw [line width=2, reverse arrow=.5] (2.center) to (5);
%	\end{pgfonlayer}
%\end{tikzpicture}
%\end{center} 
Explicit calculations for its meaning give:
\begin{align*}
(\epsilon^l_N \tns 1_S \tns \epsilon^r_N) (&\denst{lions} \tns \denst{eat} \tns \denst{meat}) \\
&= \braket{\vec{lions}}{\vec{sloths}}^2 \braket{\vec{plants}}{\vec{meat}}^2 + \\
&\hspace{10mm}\braket{\vec{lions}}{\vec{sloths}} \braket{\vec{lions}}{\vec{lions}} \braket{\vec{meat}}{\vec{meat}} 
\braket{\vec{plants}}{\vec{meat}} +\\
&\hspace{10mm} \braket{\vec{lions}}{\vec{lions}} \braket{\vec{lions}}{\vec{sloths}} \braket{\vec{meat}}{\vec{meat}}
\braket{\vec{plants}}{\vec{meat}} + \\
&\hspace{10mm} \braket{\vec{lions}}{\vec{lions}}^2 \braket{\vec{meat}}{\vec{meat}}^2 \\
&= 1
\end{align*}
\paragraph{``Sloths eat meat''}. This sentence has a very similar calculation to the one above with the resulting meaning:

\[(\epsilon^l_N \tns 1_S \tns \epsilon^r_N) (\denst{sloths} \tns \denst{eat} \tns \denst{meat}) = 0 \]

\paragraph{``Mammals eat meat''}. This sentence has the following meaning calculation:
\begin{align*}
&(\epsilon^l_N \tns 1_S \tns \epsilon^r_N) (\denst{mammals} \tns \denst{eat} \tns \denst{meat}) \\
&= (\epsilon^l_N \tns 1_S \tns \epsilon^r_N) ((\frac{1}{2}\denst{lions}+\frac{1}{2}\denst{sloths}) \tns \denst{eat} \tns \denst{meat}) \\
&=  \frac{1}{2}(\epsilon^l_N \tns 1_S \tns \epsilon^r_N)(\denst{lions} \tns \denst{eat} \tns \denst{meat}) +  \frac{1}{2}(\epsilon^l_N \tns 1_S \tns \epsilon^r_N)(\denst{sloths} \tns \denst{eat} \tns \denst{meat}) \\
&= \frac{1}{2}
\end{align*}
 
The resulting meaning of  this sentence is a mixture of ``lions eat meat'', which is true, and
``sloths eat meat'' which is false. Thus the value $1/2$ can be interpreted as being neither completely
true or completely false: the sentence ``mammals eat meat''
is true for certain mammals and false for others. 
 
\subsection{Entailment between sentences in two dimensional truth theoretic space} 

The two dimensional truth theoretic space is set as follows:
\[ true \equiv \ket{0} \equiv \left( \begin{array}{c}
1 \\
0\\
 \end{array} \right)
\qquad\qquad 
 false \equiv \ket{1} \equiv
 \left( \begin{array}{c}
0 \\
1\\
 \end{array} \right) \] 
 
\noindent
The corresponding   \emph{true} and \emph{false} density matrices are $\ketbra{0}{0}$ and $\ketbra{1}{1}$. 
 
 In the two dimensional  space, the representation of ``eats'' is set as follows. Let $A=\{lions, sloths\}$ and $B=\{meat,plants\}$, then
 
 \[ \denst{eat} \equiv \sum_{\substack{a_1,a_2 \in A \\ b_1,b_2 \in B}} \prjdif{a_1}{a_2} \tns \prj{x} \tns \prjdif{b_1}{b_2}  \]
where \[ \ket{x} \equiv
 \begin{cases}
 \ket{0} \text{ if } \ket{a_1}\ket{b_1}, \ket{a_2}\ket{b_2} \in \{ \ket{\vec{lions}}\ket{\vec{meat}}, \ket{\vec{sloths}}\ket{\vec{plants}} \} \\
 \ket{1} \text{ otherwise }
 \end{cases}
 \]
 The generalized matrix representation of this verb in the spirit of \cite{grefenstette2013} is:
 
\[   \left( \begin{array}{c c c c | c c c c }
1&0&0&1&0&1&1&0 \\
0&0&0&0&1&1&1&1 \\
0&0&0&0&1&1&1&1 \\
1&0&0&1&0&1&1&0
 \end{array} \right) \]
  
\paragraph{``Lions eat meat''.}The calculation for  the meaning of this sentence is almost exactly the same as 
the case of the one dimensional meaning, only the result is not the scalar that stands for $true$ but its density
matrix:

\[(\epsilon^l_N \tns 1_S \tns \epsilon^r_N) (\denst{lions} \tns \denst{eat} \tns \denst{meat}) = \ketbra{0}{0} \]

\paragraph{``Sloths eat meat''.}  Likewise, the calculation for the meaning of this sentence  returns $false$: 
\[(\epsilon^l_N \tns 1_S \tns \epsilon^r_N) (\denst{sloths} \tns \denst{eat} \tns \denst{meat}) = \ketbra{1}{1} \]

\paragraph{``Mammals eat meat".}  As we saw before, this sentence  has the meaning
that is the mixture of ``Lions eat meat'' and ``Sloths eat meat''; here,  this is expressed  as follows:
\begin{align*}
&(\epsilon^l_N \tns 1_S \tns \epsilon^r_N) (\denst{mammals} \tns \denst{eat} \tns \denst{meat}) \\
 &=  \frac{1}{2}(\epsilon^l_N \tns 1_S \tns \epsilon^r_N)(\denst{lions} \tns \denst{eat} \tns \denst{meat}) +  \frac{1}{2}(\epsilon^l_N \tns 1_S \tns \epsilon^r_N)(\denst{sloths} \tns \denst{eat} \tns \denst{meat}) \\
&= \frac{1}{2} \ketbra{1}{1} + \frac{1}{2}\ketbra{0}{0} 
\end{align*}

\noindent
So in a two dimensional truth theoretic model, ``Mammals eat meat'' give the completely mixed state 
in the sentence space, which has maximal entropy. This is equivalent to saying that we have no real knowledge 
whether mammals in general eat meat or not. Even if we are completely certain about whether individual mammals that 
span our space for ``mammals'' eat meat, this information differs uniformly within the members of the class,
so we cannot generalize. 

Already with a two dimensional truth theoretic model, the relation $\denst{lions} \prec \denst{mammals}$ 
carries over to  sentences. To see this, first note that we have

\begin{align*}
 &N( \denst{ \text{\it lions eat meat}} || \denst{\text{\it mammals eat meat}} )   = N \left( \ketbra{0}{0} \,\, \middle| \middle| \,\, \frac{1}{2}\ketbra{0}{0} + \frac{1}{2}\ketbra{1}{1} \right) \\
 &= (\ketbra{0}{0}) \log (\ketbra{0}{0}) - (\ketbra{0}{0}) \log \left( \frac{1}{2}\ketbra{0}{0} + \frac{1}{2}\ketbra{1}{1} \right) \\
 &= 1
\end{align*}
In the other direction, we have $N( \denst{ \text{\it mammals eat meat}} || \denst{\text{\it lions eat meat}} ) = \infty$,  
since the intersection of
the support of 
the first argument and the kernel of the second argument is non-trivial. 
These lead to the following  representativeness results  between   sentences:
\begin{align*}
R(\denst{ \text{\it lions eat meat}}, \denst{ \text{\it mammals eat meat}} ) &= 1/2 \\
 R(\denst{ \text{\it mammals eat meat}} || \denst{\text{\it lions eat meat}} ) &= 0 
 \end{align*}
As a result we obtain:
\[  \denst{ \text{\it lions eat meat}} \prec \denst{\text{\it mammals eat meat}} \]

Since these two sentences share the same verb phrase,  from-meaning-of-words-to-the-meaning-of-sentence map carries the hyponymy relation in the subject 
words of the respective sentences to the resulting sentence meanings. By using the density matrix representations
of word meanings together with the categorical map from the meanings of words to the meanings of sentences, 
 the knowledge that a lion is an animal lets us infer that ``mammals eat meat'' implies ``lions eat meat'':
 \[ (\denst{\text{\it lions}} \prec \denst{\text{\it mammals}}) \ra (\denst{ \text{\it lions eat meat}} \prec \denst{\text{\it mammals eat meat}}) \]

\paragraph{``Dogs eat meat''.} To see how the completely mixed state differs from a perfectly correlated but 
pure state in the context of linguistic meaning, consider a new noun $\denst{dog} = \prj{dog}$ and redefine 
eat in terms of the bases $\{ \vec{lions}, \vec{dogs} \}$ and $\{\vec{meat}, \vec{plants}\}$, so that it will reflect the fact that dogs 
eat \textit{both} meat and plants. 
We define ``eat'' so that it results in the value of being ``half-true half-false'' when it takes ``dogs'' as subject and 
``meat'' or ``plants'' as object. The value ``half-true half-false'' is the superposition of $true$ and $false$: 
$\frac{1}{2}\ket{0} + \frac{1}{2}\ket{1}$. With this assumptions,  $\denst{eat}$ will still be a pure state with the following representation  in \cat{FVect}: 
\begin{align*}
\ket{\vec{eat}} = &\ket{\vec{lions}} \tns \ket{0} \tns \ket{\vec{meat}} + 
 \ket{\vec{lions}} \tns \ket{1} \tns \ket{\vec{plants}} +\\
 &\ket{\vec{dogs}} \tns (\frac{1}{2}\ket{0} + \frac{1}{2}\ket{1}) \tns \ket{\vec{meat}} + 
 \ket{\vec{dogs}} \tns  (\frac{1}{2}\ket{0} + \frac{1}{2}\ket{1}) \tns \ket{\vec{plants}} 
\end{align*}

\noindent
Hence, the density matrix representation of ``eat''  becomes:
\[ \denst{eat} = \prj{eat}\]

\noindent
The calculation for the meaning of the sentence is as follows:
\begin{align*}
&(\epsilon^l_N \tns 1_S \tns \epsilon^r_N) (\denst{dogs} \tns \denst{eat} \tns \denst{meat})\\
&= (\epsilon^l_N \tns 1_S \tns \epsilon^r_N) (\prj{dogs} \tns \prj{eat} \tns \prj {meat}) \\
&= (\frac{1}{2}\ket{0} + \frac{1}{2}\ket{1})(\frac{1}{2}\bra{0} + \frac{1}{2}\bra{1})
\end{align*}

\noindent
So in this case, we are certain that it  is half-true and half-false that dogs eat meat. This is in contrast with 
the completely mixed state we got from ``Mammals eat meat'', for which the truth or falsity of the sentence
was entirely unknown.

\paragraph{``Mammals eat meat'', again.}  Let ``mammals'' now be defined as:
\[ \denst{mammals} = \frac{1}{2}\denst{lions} + \frac{1}{2} \denst{dogs} \] 
The calculation for the meaning of this sentence gives:
\begin{align*}
&(\epsilon^l_N \tns 1_S \tns \epsilon^r_N) (\denst{mammals} \tns \denst{eat} \tns \denst{meat}) \\
 &=  \frac{1}{2}(\epsilon^l_N \tns 1_S \tns \epsilon^r_N)(\denst{lions} \tns \denst{eat} \tns \denst{meat}) +  \frac{1}{2}(\epsilon^l_N \tns 1_S \tns \epsilon^r_N)(\denst{dogs} \tns \denst{eat} \tns \denst{meat}) \\
%&= \frac{1}{2} \ketbra{0}{0} + \frac{1}{2}((\frac{1}{2}\ket{0} + \frac{1}{2}\ket{1})(\frac{1}{2}\bra{0} + \frac{1}{2}\bra{1})) \\
&= \frac{3}{4} \ketbra{0}{0} + \frac{1}{4}\ketbra{0}{1} + \frac{1}{4}\ketbra{1}{0} + \frac{1}{4}\ketbra{1}{1} \\
\end{align*}

This time the resulting sentence representation is not completely mixed. This means that we can 
generalize the knowledge we 
have from the specific instances of mammals to the entire class to some extent,  but 
still we cannot generalize completely. This is a mixed state, which indicates that 
even if the sentence is closer to $true$ than to $false$, the degree of truth isn't homogeneous throughout the 
elements of the class. The non-zero non-diagonals indicate that it is also partially correlated, 
which means that there are some instances of ``mammals'' for which this sentence is true to a degree, 
but not completely. The relative similarity measures between  $true$ and $false$ and the sentence can be calculated
explicitly using fidelity:
\[
F \big(  \ketbra{1}{1} , \denst{\text{\it mammals eat meat} } \big)
= \bra{1} \denst{\text{\it mammals eat meat} } \ket{1} 
= \frac{1}{4} \]
\[
F \big(  \ketbra{0}{0}  , \denst{\text{\it mammals eat meat} } \big) 
= \bra{0} \denst{\text{\it mammals eat meat} } \ket{0} 
= \frac{3}{4}
\]

Notice that these values are different from the values for the representativeness for truth and falsity of 
the sentence, even thought they are proportional: the more representative their density matrices, 
the more similar the sentences are to each other. For example, we have:
\begin{align*}
N \big( & \ketbra{1}{1} \,\, \| \,\, \denst{\text{\it mammals eat meat} } \big) \\ 
&= \tr \big( \ketbra{1}{1}) \log (\ketbra{1}{1}) \big) - \tr \big( \ketbra{1}{1} \log ( \frac{3}{4} \ketbra{0}{0} + \frac{1}{4}\ketbra{0}{1} + \frac{1}{4}\ketbra{1}{0} + \frac{1}{4}\ketbra{1}{1} ) \big)\\
&\approx 2 
\end{align*}
Hence, $R \big( \ketbra{1}{1} \,\, \| \,\, \denst{\text{\it mammals eat meat} } \big) \approx .33$. On the other hand: 
\begin{align*}
N \big( & \ketbra{0}{0} \,\, \| \,\, \denst{\text{\it mammals eat meat} } \big) \\ 
&= \tr \big( \ketbra{0}{0}) \log (\ketbra{0}{0}) \big) - \tr \big( \ketbra{0}{0} \log ( \frac{3}{4} \ketbra{0}{0} + \frac{1}{4}\ketbra{0}{1} + \frac{1}{4}\ketbra{1}{0} + \frac{1}{4}\ketbra{1}{1} ) \big)\\
&\approx 0.41
\end{align*}
Hence, $R\big(  \ketbra{0}{0} \,\, \| \,\, \denst{\text{\it mammals eat meat} } \big) \approx 0.71$

\section{A Distributional Example}

The goal of this section is to show how one can obtain  density matrices for words using lexical taxonomies and co-occurrence  frequencies counted from corpora of text. We  show how  these density matrices are used in  example sentences and how the density matrices of their meanings look like. We compute the  representativeness formula for these sentences to provide a proof of concept that this measure does   makes sense   for data harvested from corpora distributionally and that  its application is not restricted to truth-theoretic models.    Implementing these constructions on real data and validating them on large scale datasets constitute work in progress.

\subsection{Entailment between nouns}  
Suppose we have a noun space $N$. Let  the subspace relevant for this part of the example be spanned by lemmas {\it pub, pitcher,  tonic}. Assume that the  (non-normalized version of the) vectors of the atomic words {\it lager} and {\it ale} in this subspace are as follows:
\[
\ov{{lager}} = 6 \times \ov{{pub}} +  5  \times\ov{{pitcher}} + 0 \times \ov{{tonic}} \qquad
\ov{{ale}} =  7  \times \ov{{pub}} +  3 \times \ov{{pitcher}} + 0 \times \ov{{tonic}}
\]
Suppose further that we are given taxonomies such as `beer = lager + ale', harvested from a resource such as WordNet. Atomic words (i.e. leafs of the taxonomy),  correspond to \emph{pure} states and their density matrices are the projections onto the one dimensional subspace spanned by $\ket{\overrightarrow{w}}\bra{\overrightarrow{w}}$.  Non-atomic words (such as {\it beer}) are also density matrices,  harvested from the corpus using a feature-based method similar to that of  \cite{geffet2005}. This  is done by  counting (and normalising)  the frequency of times a word has co-occurred with a subset  $B$  of  bases in a window in which other bases (the ones not in $B$) have not occurred.   

Formally, for a subset of bases $\{b1, b2,..., bn\}$, we collect  co-ordinates $C_{ij}$  for each tuple $\ket{bi} \ket{bj}$ and build the density matrix $\sum_{ij} C_{ij} \ket{bi} \ket{bj}$.  

%We denote all such density matrices as before, that is by putting a hat on top their corresponding word  $\denst{\textmd{w}}$. 

For example, suppose  we see {\it beer}  six times with just {\it pub}, seven times with both {\it pub} and {\it pitcher}, and none-whatsoever with {\it tonic}. Its corresponding  density matrix will be as follows: 
\[  \denst{{beer}} = 6 \times \prj{pub} + 7 \times ( \ket{\vec{pub}} + \ket{\vec{pitcher}})(\bra{\vec{pub}} + \bra{\vec{pitcher}}) \]
\[= 13 \times \prj{pub} +  7  \times \prjdif{pub}{pitcher} + 7 \times \prjdif{pitcher}{pub} +
7 \times \prj{pitcher}
 \]
 
 \smallskip
 To calculate  the similarity and representativeness of the word pairs, we first normalize them via the operation  $\frac{\rho}{\text{Tr}\rho}$, then apply the corresponding formulae.  For example, the degree of similarity between `beer' and `lager'  using fidelity is  as follows:
\[
\text{Tr} \sqrt{ \denst{{lager}}^{\frac{1}{2}} \cdot \denst{{beer}} \cdot\denst{{lager}}^{\frac{1}{2}}} = 0.93
 \]
 The degree  of entailment  $lager \prec  beer $ is 0.82 as   computed as follows:
\[\frac{1}{1+ \text{Tr}(\denst{{lager}} \cdot \log (\denst{{lager}} )- \denst{{lager}} \cdot \log (\denst{{beer}}))} = 0.82 \] 
The degree of entailment   $beer \prec  lager$ is $0$,  like one would expect.
 
\subsection{Entailment between sentences}
To see how the entailment between sentences follows from the entailment between words, consider  example  sentences \textit{`Psychiatrist is drinking lager'} and \textit{`Doctor is drinking beer'}. For the sake of brevity, we assume the meanings of {\it psychiatrist} and {\it doctor}  are mixtures of basis elements, as follows:
\begin{align*}
\denst{{psychiatrist}} &= 2 \times \prj{patient} + 5\times \prj{mental} \\
\denst{{doctor}} &= 5\times  \prj{patient} + 2\times \prj{mental} + 3 \times \prj{surgery}
\end{align*}
The similarity between {\it psychiatrist} and {\it doctor} is: 
\[
 S(\denst{{psychiatrist}}, \denst{{doctor}}) =  S(\denst{{doctor}}, \denst{{psychiatrist}}) = 0.76 
 \]
The representativeness between them is:
\[
  R(\denst{{psychiatrist}}, \denst{{doctor}}) = 0.49  \qquad \qquad
 R(\denst{{doctor},} \denst{{psychiatrist}}) = 0
\]
We  build  matrices for  the verb {\it drink}  following the method of \cite{grefenstette2011}. Intuitively this is as follows:  the value in entry $(i,j)$ of this matrix will reflect how typical it is for the verb to have a subject related to the $i$th basis and an object related to the $j$th basis. We assume that the  small part of the matrix that interests us for this example is as follows:

\begin{center}
\begin{tabular}{ |c||c|c|c| } 
 \hline
\emph{drink}& {\it pub} & {\it pitcher} & {\it tonic} \\  \hline \hline
{\it patient} &4  & 5 & 3 \\ \hline
 {\it mental} & 6 & 3 & 2 \\ \hline
 {\it surgery} & 1  & 2 & 1 \\  \hline
\end{tabular}
\end{center}

 This representation can be seen as a pure state living in a second order tensor. Therefore the density matrix representation of the same object is $\denst{\text{\it drink}} = \prj{drink}$, a fourth order tensor. Lifting the simplifications introduced in \cite{grefenstette2011} from vectors to density matrices, we obtain the following linear algebraic closed forms for the meaning of the sentences:
 \begin{align*}
 \denst{\text{\it Psychiatrist is drinking lager}} &= \denst{\text{\it drink}} \odot (\denst{\text{\it psychiatrist}} \otimes \denst{\text{\it lager}}) \\
  \denst{\text{\it Doctor is drinking beer}} &= \denst{\text{\it drink}} \odot (\denst{\text{\it doctor}} \otimes \denst{\text{\it beer}})
 \end{align*}
 
 \noindent
Applying the  fidelity and representativeness formulae to  sentence representations, we obtain the following values:
\begin{align*}
S(\denst{\text{\it Psychiatrist is drinking lager}},\denst{\text{\it Doctor is drinking beer}}) &= 0.81\\
R(\denst{\text{\it Psychiatrist is drinking lager}},\denst{\text{\it Doctor is drinking beer}}) &= 0.53 \\
R(\denst{\text{\it Doctor is drinking beer}},\denst{\text{\it Psychiatrist is drinking lager}}) &= 0
\end{align*}
From the relations $\text{\it psychiatrist} \prec \text{\it doctor}$ and $\text{\it lager} \prec \text{\it beer}$ we obtain the desired entailment between sentences:
\[
\text{\it Psychiatrist is drinking lager} \prec \text{\it Doctor is drinking beer}\,.
\]
The entailment between these two sentences follows from the entailment  between their subjects and the entailment between their objects. In the examples that we have considered so far, the verbs of sentences are  the same. This is not a necessity. One can  have  entailment between sentences that do not have the same verbs, but where the verbs entail each other, examples can be found in \cite{balkir2014}. The reason we do not  present such cases  here is lack of space.

\section{Conclusion and Future Work}
The often stated long term goal of compositional distributional models is to merge distributional and formal semantics. However, what formal and distributional semantics \textit{do} with the resulting meaning representations is quite different. Distributional semanticists care about \textit{similarity} while formal semanticists aim to capture \textit{truth} and \textit{inference}. In this work we presented a theory of meaning using basic objects that will not confine us to the realm of only distributional or only formal semantics. The immediate next step is to develop methods for obtaining density matrix representations of words from corpus, that are more robust to statistical noise, and testing the usefulness of the theory in large scale experiments. 

The problem of integrating function words such as `and', `or', `not', `every' into a distributional setting has been notoriously hard. We hope that the characterization of compositional distributional entailment on these very simple types of sentences will provide a foundation on which we can define representations of these function words, and develop a more logical theory of compositional distributional meaning.

\bibliographystyle{plain}
\bibliography{qpl15.bib}

\end{document}

%% file: book-header.tex
\usepackage{docmute}
\usepackage{amsmath,amsfonts,amssymb,mathtools,xspace}
\usepackage{tikz}
\usetikzlibrary{decorations.pathreplacing}
\usetikzlibrary{decorations.markings}
\usetikzlibrary{calc}
\usepackage{etoolbox}
\usepackage{datetime}

\newenvironment{pic}[1][]
{\begin{aligned}\begin{tikzpicture}[#1]}
{\end{tikzpicture}\end{aligned}}
\newcommand{\edges}[1][]%
{%\end{scope}\end{pgfonlayer}\begin{pgfonlayer}{foreground}\begin{scope}[#1]
}

\makeatletter
\def\calign@preamble{%
   &\hfil\strut@
    \setboxz@h{\@lign$\m@th\displaystyle{##}$}%
    \ifmeasuring@\savefieldlength@\fi
    \set@field
    \hfil
    \tabskip\alignsep@
}
\let\cmeasure@\measure@
\patchcmd\cmeasure@{\divide\@tempcntb\tw@}{}{}{}
\patchcmd\cmeasure@{\divide\@tempcntb\tw@}{}{}{}
\patchcmd\cmeasure@{\ifodd\maxfields@
  \global\advance\maxfields@\@ne
  \fi}{}{}{}    
  
\makeatother


\newcommand\tinymatrix[1]
{\left( \hspace{-2pt} \renewcommand\thickspace{\kern2pt} \scriptstyle\begin{smallmatrix} #1 \end{smallmatrix} \hspace{-2pt} \right)}

\newcommand\ignore[1]{}

\pgfdeclarelayer{foreground}
\pgfdeclarelayer{background}
\pgfsetlayers{main,foreground,background}

\makeatletter
\pgfkeys{%
  /tikz/on layer/.code={
    \pgfonlayer{#1}\begingroup
    \aftergroup\endpgfonlayer
    \aftergroup\endgroup
  },
  /tikz/node on layer/.code={
    \gdef\node@@on@layer{%
      \setbox\tikz@tempbox=\hbox\bgroup\pgfonlayer{#1}\unhbox\tikz@tempbox\endpgfonlayer\egroup}
    \aftergroup\node@on@layer
  },
  /tikz/end node on layer/.code={
    \endpgfonlayer\endgroup\endgroup
  }
}
\def\node@on@layer{\aftergroup\node@@on@layer}
\makeatother\def\thickness{0.7pt}
\tikzstyle{oldmorphism}=[minimum width=30pt, minimum height=16pt, draw, font=\small, inner sep=0pt, fill=white, line width=\thickness]
\tikzstyle{cross}=[preaction={draw=white, -, line width=10pt}]
\tikzstyle{braid}=[double=black, line width=3*\thickness, double distance=\thickness, white]
\tikzstyle{string}=[line width=\thickness]
\tikzstyle{scalar}=[circle, inner sep=0pt, minimum width=15pt, draw, line width=\thickness]
\tikzstyle{dot}=[circle, draw=black, fill=black!25, inner sep=.4ex, line width=\thickness, node on layer=foreground]
\tikzstyle{blackdot}=[circle, draw=black, fill=black!50, inner sep=.4ex, line width=\thickness, node on layer=foreground]
\tikzstyle{whitedot}=[circle, draw=black, fill=white, inner sep=.4ex, line width=\thickness, node on layer=foreground]
\tikzstyle{mixedmorphism}=[morphism, minimum width=30pt, minimum height=16pt, draw, font=\small, inner sep=0pt, fill=white, line width=\thickness,rounded corners=1ex]
\tikzstyle{thick}=[line width=\thickness]
\tikzstyle{tiny}=[font=\tiny]

\tikzset{arrow/.style={decoration={
    markings,
    mark=at position #1 with \arrow{thickarrow}},
    postaction=decorate}
}
\tikzset{reverse arrow/.style={decoration={
    markings,
    mark=at position #1 with \arrow{reversethickarrow}},
    postaction=decorate}
}

\newif\ifblack\pgfkeys{/tikz/black/.is if=black}
\newif\ifwedge\pgfkeys{/tikz/wedge/.is if=wedge}
\newif\ifvflip\pgfkeys{/tikz/vflip/.is if=vflip}
\newif\ifhflip\pgfkeys{/tikz/hflip/.is if=hflip}
\newif\ifhvflip\pgfkeys{/tikz/hvflip/.is if=hvflip}
\newif\ifconnectnw\pgfkeys{/tikz/connect nw/.is if=connectnw}
\newif\ifconnectne\pgfkeys{/tikz/connect ne/.is if=connectne}
\newif\ifconnectsw\pgfkeys{/tikz/connect sw/.is if=connectsw}
\newif\ifconnectse\pgfkeys{/tikz/connect se/.is if=connectse}
\newif\ifconnectn\pgfkeys{/tikz/connect n/.is if=connectn}
\newif\ifconnects\pgfkeys{/tikz/connect s/.is if=connects}
\newif\ifconnectnwf\pgfkeys{/tikz/connect nw >/.is if=connectnwf}
\newif\ifconnectnef\pgfkeys{/tikz/connect ne >/.is if=connectnef}
\newif\ifconnectswf\pgfkeys{/tikz/connect sw >/.is if=connectswf}
\newif\ifconnectsef\pgfkeys{/tikz/connect se >/.is if=connectsef}
\newif\ifconnectnf\pgfkeys{/tikz/connect n >/.is if=connectnf}
\newif\ifconnectsf\pgfkeys{/tikz/connect s >/.is if=connectsf}
\newif\ifconnectnwr\pgfkeys{/tikz/connect nw </.is if=connectnwr}
\newif\ifconnectner\pgfkeys{/tikz/connect ne </.is if=connectner}
\newif\ifconnectswr\pgfkeys{/tikz/connect sw </.is if=connectswr}
\newif\ifconnectser\pgfkeys{/tikz/connect se </.is if=connectser}
\newif\ifconnectnr\pgfkeys{/tikz/connect n </.is if=connectnr}
\newif\ifconnectsr\pgfkeys{/tikz/connect s </.is if=connectsr}
\tikzset{keylengthnw/.initial=\connectheight}
\tikzset{keylengthn/.initial =\connectheight}
\tikzset{keylengthne/.initial=\connectheight}
\tikzset{keylengthsw/.initial=\connectheight}
\tikzset{keylengths/.initial =\connectheight}
\tikzset{keylengthse/.initial=\connectheight}
\tikzset{connect nw length/.style={connect nw=true, keylengthnw={#1}}}
\tikzset{connect n length/.style ={connect n =true, keylengthn ={#1}}}
\tikzset{connect ne length/.style={connect ne=true, keylengthne={#1}}}
\tikzset{connect sw length/.style={connect sw=true, keylengthsw={#1}}}
\tikzset{connect s length/.style ={connect s =true, keylengths ={#1}}}
\tikzset{connect se length/.style={connect se=true, keylengthse={#1}}}
\tikzset{connect nw < length/.style={connect nw <=true, keylengthnw={#1}}}
\tikzset{connect n < length/.style ={connect n <=true,  keylengthn ={#1}}}
\tikzset{connect ne < length/.style={connect ne <=true, keylengthne={#1}}}
\tikzset{connect sw < length/.style={connect sw <=true, keylengthnw={#1}}}
\tikzset{connect s < length/.style ={connect s <=true,  keylengths ={#1}}}
\tikzset{connect se < length/.style={connect se <=true, keylengthse={#1}}}
\tikzset{connect nw > length/.style={connect nw >=true, keylengthnw={#1}}}
\tikzset{connect n > length/.style ={connect n >=true,  keylengthn ={#1}}}
\tikzset{connect ne > length/.style={connect ne >=true, keylengthne={#1}}}
\tikzset{connect sw > length/.style={connect sw >=true, keylengthsw={#1}}}
\tikzset{connect s > length/.style ={connect s >=true,  keylengths ={#1}}}
\tikzset{connect se > length/.style={connect se >=true, keylengthse={#1}}}

\newlength\morphismheight
\newlength\minimummorphismwidth
\newlength\stateheight
\newlength\minimumstatewidth
\newlength\connectheight
\tikzset{width/.initial=\minimummorphismwidth}

\makeatletter
\pgfarrowsdeclare{thickarrow}{thickarrow}
{
  \pgfutil@tempdima=-0.84pt%
  \advance\pgfutil@tempdima by-1.3\pgflinewidth%
  \pgfutil@tempdimb=-1.7pt%
  \advance\pgfutil@tempdimb by.625\pgflinewidth%
  \pgfarrowsleftextend{+\pgfutil@tempdima}
  \pgfarrowsrightextend{+\pgfutil@tempdimb}
}
{
  \pgfmathparse{\pgfgetarrowoptions{thickarrow}}%
  \pgfsetlinewidth{1.25 pt}
  \pgfutil@tempdima=0.28pt%
  \advance\pgfutil@tempdima by.3\pgflinewidth%
  \pgfsetlinewidth{0.8\pgflinewidth}
  \pgfsetdash{}{+0pt}
  \pgfsetroundcap
  \pgfsetroundjoin
  \pgfpathmoveto{\pgfqpoint{-3\pgfutil@tempdima}{4\pgfutil@tempdima}}
  \pgfpathcurveto
  {\pgfqpoint{-2.75\pgfutil@tempdima}{2.5\pgfutil@tempdima}}
  {\pgfqpoint{0pt}{0.25\pgfutil@tempdima}}
  {\pgfqpoint{0.75\pgfutil@tempdima}{0pt}}
  \pgfpathcurveto
  {\pgfqpoint{0pt}{-0.25\pgfutil@tempdima}}
  {\pgfqpoint{-2.75\pgfutil@tempdima}{-2.5\pgfutil@tempdima}}
  {\pgfqpoint{-3\pgfutil@tempdima}{-4\pgfutil@tempdima}}
  \pgfusepathqstroke
}
\pgfarrowsdeclare{reversethickarrow}{reversethickarrow}
{
  \pgfutil@tempdima=-0.84pt%
  \advance\pgfutil@tempdima by-1.3\pgflinewidth%
  \pgfutil@tempdimb=0.2pt%
  \advance\pgfutil@tempdimb by.625\pgflinewidth%
  \pgfarrowsleftextend{+\pgfutil@tempdima}
  \pgfarrowsrightextend{+\pgfutil@tempdimb}
}
{
  \pgftransformxscale{-1}
  \pgfmathparse{\pgfgetarrowoptions{thickarrow}}%
  \ifpgfmathunitsdeclared%
    \pgfmathparse{\pgfmathresult pt}%
  \else%  
    \pgfmathparse{\pgfmathresult*\pgflinewidth}%
  \fi%
  \let\thickness=\pgfmathresult
  \pgfsetlinewidth{1.25 pt}
  \pgfutil@tempdima=0.28pt%
  \advance\pgfutil@tempdima by.3\pgflinewidth%
  \pgfsetlinewidth{0.8\pgflinewidth}
  \pgfsetdash{}{+0pt}
  \pgfsetroundcap
  \pgfsetroundjoin
  \pgfpathmoveto{\pgfqpoint{-3\pgfutil@tempdima}{4\pgfutil@tempdima}}
  \pgfpathcurveto
  {\pgfqpoint{-2.75\pgfutil@tempdima}{2.5\pgfutil@tempdima}}
  {\pgfqpoint{0pt}{0.25\pgfutil@tempdima}}
  {\pgfqpoint{0.75\pgfutil@tempdima}{0pt}}
  \pgfpathcurveto
  {\pgfqpoint{0pt}{-0.25\pgfutil@tempdima}}
  {\pgfqpoint{-2.75\pgfutil@tempdima}{-2.5\pgfutil@tempdima}}
  {\pgfqpoint{-3\pgfutil@tempdima}{-4\pgfutil@tempdima}}
  \pgfusepathqstroke
}
\makeatother

\makeatletter
\pgfdeclareshape{ground}
{
    \savedanchor\centerpoint
    {
        \pgf@x=0pt
        \pgf@y=0pt
    }
    \anchor{center}{\centerpoint}
    \anchorborder{\centerpoint}
    \anchor{north}
    {
        \pgf@x=0pt
        \pgf@y=0.5*0.33*\stateheight
    }
    \anchor{south}
    {
        \pgf@x=0pt
        \pgf@y=0pt
    }
    \saveddimen\overallwidth
    {
        \pgfkeysgetvalue{/pgf/minimum width}{\minwidth}
        \pgf@x=\minimumstatewidth
        \ifdim\pgf@x<\minwidth
            \pgf@x=\minwidth
        \fi
    }
    \backgroundpath
    {
        \begin{pgfonlayer}{foreground}
        \pgfsetstrokecolor{black}
        \pgfsetlinewidth{1.25pt}
        \ifhflip
            \pgftransformyscale{-1}
        \fi
        \pgftransformscale{0.5}
        \pgfpathmoveto{\pgfpoint{-0.5*\overallwidth}{0}}
        \pgfpathlineto{\pgfpoint{0.5*\overallwidth}{0}}
        \pgfpathmoveto{\pgfpoint{-0.33*\overallwidth}{0.33*\stateheight}}
        \pgfpathlineto{\pgfpoint{0.33*\overallwidth}{0.33*\stateheight}}
        \pgfpathmoveto{\pgfpoint{-0.16*\overallwidth}{0.66*\stateheight}}
        \pgfpathlineto{\pgfpoint{0.16*\overallwidth}{0.66*\stateheight}}
        \pgfpathmoveto{\pgfpoint{-0.02*\overallwidth}{\stateheight}}
        \pgfpathlineto{\pgfpoint{0.02*\overallwidth}{\stateheight}}
        \pgfusepath{stroke}
        \end{pgfonlayer}
    }
}
\tikzset{forward arrow style/.style={every to/.style, decoration={
    markings,
    mark=at position 0.5 with \arrow{thickarrow}},
    postaction=decorate}}
\tikzset{reverse arrow style/.style={every to/.style, decoration={
    markings,
    mark=at position 0.5 with \arrow{reversethickarrow}},
    postaction=decorate}}
\pgfdeclareshape{morphism}
{
    \savedanchor\centerpoint
    {
        \pgf@x=0pt
        \pgf@y=0pt
    }
    \anchor{center}{\centerpoint}
    \anchorborder{\centerpoint}
    \saveddimen\savedlengthnw
    {
        \pgfkeysgetvalue{/tikz/keylengthnw}{\len}
        \pgf@x=\len
    }
    \saveddimen\savedlengthn
    {
        \pgfkeysgetvalue{/tikz/keylengthn}{\len}
        \pgf@x=\len
    }
    \saveddimen\savedlengthne
    {
        \pgfkeysgetvalue{/tikz/keylengthne}{\len}
        \pgf@x=\len
    }
    \saveddimen\savedlengthsw
    {
        \pgfkeysgetvalue{/tikz/keylengthsw}{\len}
        \pgf@x=\len
    }
    \saveddimen\savedlengths
    {
        \pgfkeysgetvalue{/tikz/keylengths}{\len}
        \pgf@x=\len
    }
    \saveddimen\savedlengthse
    {
        \pgfkeysgetvalue{/tikz/keylengthse}{\len}
        \pgf@x=\len
    }
    \saveddimen\overallwidth
    {
        \pgfkeysgetvalue{/tikz/width}{\minwidth}
        \pgf@x=\wd\pgfnodeparttextbox
        \ifdim\pgf@x<\minwidth
            \pgf@x=\minwidth
        \fi
    }
    \savedanchor{\upperrightcorner}
    {
        \pgf@y=.5\ht\pgfnodeparttextbox
        \advance\pgf@y by -.5\dp\pgfnodeparttextbox
        \pgf@x=.5\wd\pgfnodeparttextbox
    }
    \anchor{north}
    {
        \pgf@x=0pt
        \pgf@y=0.5\morphismheight
    }
    \anchor{north east}
    {
        \pgf@x=\overallwidth
        \multiply \pgf@x by 2
        \divide \pgf@x by 5
        \pgf@y=0.5\morphismheight
    }
    \anchor{east}
    {
        \pgf@x=\overallwidth
        \divide \pgf@x by 2
        \advance \pgf@x by 5pt
        \pgf@y=0pt
    }
    \anchor{west}
    {
        \pgf@x=-\overallwidth
        \divide \pgf@x by 2
        \advance \pgf@x by -5pt
        \pgf@y=0pt
    }
    \anchor{north west}
    {
        \pgf@x=-\overallwidth
        \multiply \pgf@x by 2
        \divide \pgf@x by 5
        \pgf@y=0.5\morphismheight
    }
    \anchor{connect nw}
    {
        \pgf@x=-\overallwidth
        \multiply \pgf@x by 2
        \divide \pgf@x by 5
        \pgf@y=0.5\morphismheight
        \advance\pgf@y by \savedlengthnw
    }
    \anchor{connect ne}
    {
        \pgf@x=\overallwidth
        \multiply \pgf@x by 2
        \divide \pgf@x by 5
        \pgf@y=0.5\morphismheight
        \advance\pgf@y by \savedlengthne
    }
    \anchor{connect sw}
    {
        \pgf@x=-\overallwidth
        \multiply \pgf@x by 2
        \divide \pgf@x by 5
        \pgf@y=-0.5\morphismheight
        \advance\pgf@y by -\savedlengthsw
    }
    \anchor{connect se}
    {
        \pgf@x=\overallwidth
        \multiply \pgf@x by 2
        \divide \pgf@x by 5
        \pgf@y=-0.5\morphismheight
        \advance\pgf@y by -\savedlengthse
    }
    \anchor{connect n}
    {
        \pgf@x=0pt
        \pgf@y=0.5\morphismheight
        \advance\pgf@y by \savedlengthn
    }
    \anchor{connect s}
    {
        \pgf@x=0pt
        \pgf@y=-0.5\morphismheight
        \advance\pgf@y by -\savedlengths
    }
    \anchor{south east}
    {
        \pgf@x=\overallwidth
        \multiply \pgf@x by 2
        \divide \pgf@x by 5
        \pgf@y=-0.5\morphismheight
    }
    \anchor{south west}
    {
        \pgf@x=-\overallwidth
        \multiply \pgf@x by 2
        \divide \pgf@x by 5
        \pgf@y=-0.5\morphismheight
    }
    \anchor{south}
    {
        \pgf@x=0pt
        \pgf@y=-0.5\morphismheight
    }
    \anchor{text}
    {
        \upperrightcorner
        \pgf@x=-\pgf@x
        \pgf@y=-\pgf@y
    }
    \backgroundpath
    {
        \pgfsetstrokecolor{black}
        \pgfsetlinewidth{\thickness}
        \begin{scope}
                \ifhflip
                    \pgftransformyscale{-1}
                \fi
                \ifvflip
                    \pgftransformxscale{-1}
                \fi
                \ifhvflip
                    \pgftransformxscale{-1}
                    \pgftransformyscale{-1}
                \fi
                \pgfpathmoveto{\pgfpoint
                    {-0.5*\overallwidth-5pt}
                    {0.5*\morphismheight}}
                \pgfpathlineto{\pgfpoint
                    {0.5*\overallwidth+5pt}
                    {0.5*\morphismheight}}
                \ifwedge
                    \pgfpathlineto{\pgfpoint
                        {0.5*\overallwidth + 15pt}
                        {-0.5*\morphismheight}}
                \else
                    \pgfpathlineto{\pgfpoint
                        {0.5*\overallwidth + 5pt}
                        {-0.5*\morphismheight}}
                \fi
                \pgfpathlineto{\pgfpoint
                    {-0.5*\overallwidth-5pt}
                    {-0.5*\morphismheight}}
                \pgfpathclose
                \pgfusepath{stroke}
        \end{scope}
        \ifconnectnw
            \pgfpathmoveto{\pgfpoint
                {-0.4*\overallwidth}
                {0.5*\morphismheight}}
            \pgfpathlineto{\pgfpoint
                {-0.4*\overallwidth}
                {0.5*\morphismheight+\savedlengthnw}}
            \pgfusepath{stroke}
        \fi
        \ifconnectne
            \pgfpathmoveto{\pgfpoint
                {0.4*\overallwidth}
                {0.5*\morphismheight}}
            \pgfpathlineto{\pgfpoint
                {0.4*\overallwidth}
                {0.5*\morphismheight+\savedlengthne}}
            \pgfusepath{stroke}
        \fi
        \ifconnectsw
            \pgfpathmoveto{\pgfpoint
                {-0.4*\overallwidth}
                {-0.5*\morphismheight}}
            \pgfpathlineto{\pgfpoint
                {-0.4*\overallwidth}
                {-0.5*\morphismheight-\savedlengthsw}}
            \pgfusepath{stroke}
        \fi
        \ifconnectse
            \pgfpathmoveto{\pgfpoint
                {0.4*\overallwidth}
                {-0.5*\morphismheight}}
            \pgfpathlineto{\pgfpoint
                {0.4*\overallwidth}
                {-0.5*\morphismheight-\savedlengthse}}
            \pgfusepath{stroke}
        \fi
        \ifconnectn
            \pgfpathmoveto{\pgfpoint
                {0pt}
                {0.5*\morphismheight}}
            \pgfpathlineto{\pgfpoint
                {0pt}
                {0.5*\morphismheight+\savedlengthn}}
            \pgfusepath{stroke}
        \fi
        \ifconnects
            \pgfpathmoveto{\pgfpoint
                {0pt}
                {-0.5*\morphismheight}}
            \pgfpathlineto{\pgfpoint
                {0pt}
                {-0.5*\morphismheight-\savedlengths}}
            \pgfusepath{stroke}
        \fi
        \ifconnectnwf
            \draw [forward arrow style] (-0.4*\overallwidth,0.5*\morphismheight)
                to (-0.4*\overallwidth,0.5*\morphismheight+\savedlengthnw);
        \fi
        \ifconnectnef
            \draw [forward arrow style] (0.4*\overallwidth,0.5*\morphismheight)
                to (0.4*\overallwidth,0.5*\morphismheight+\savedlengthne);
        \fi
        \ifconnectswf
            \draw [forward arrow style] (-0.4*\overallwidth,-0.5*\morphismheight-\savedlengthsw)
                to (-0.4*\overallwidth,-0.5*\morphismheight);
        \fi
        \ifconnectsef
            \draw [forward arrow style] (0.4*\overallwidth,-0.5*\morphismheight-\savedlengthse)
                to (0.4*\overallwidth,-0.5*\morphismheight);
        \fi
        \ifconnectnf
            \draw [forward arrow style] (0,0.5*\morphismheight)
                to (0,0.5*\morphismheight+\savedlengthn);
        \fi
        \ifconnectsf
            \draw [forward arrow style] (0,-0.5*\morphismheight-\savedlengths)
                to (0,-0.5*\morphismheight);
        \fi
        \ifconnectnwr
            \draw [reverse arrow style] (-0.4*\overallwidth,0.5*\morphismheight)
                to (-0.4*\overallwidth,0.5*\morphismheight+\savedlengthnw);
        \fi
        \ifconnectner
            \draw [reverse arrow style] (0.4*\overallwidth,0.5*\morphismheight)
                to (0.4*\overallwidth,0.5*\morphismheight+\savedlengthne);
        \fi
        \ifconnectswr
            \draw [reverse arrow style] (-0.4*\overallwidth,-0.5*\morphismheight-\savedlengthsw)
                to (-0.4*\overallwidth,-0.5*\morphismheight);
        \fi
        \ifconnectser
            \draw [reverse arrow style] (0.4*\overallwidth,-0.5*\morphismheight-\savedlengthse)
                to (0.4*\overallwidth,-0.5*\morphismheight);
        \fi
        \ifconnectnr
            \draw [reverse arrow style] (0,0.5*\morphismheight)
                to (0,0.5*\morphismheight+\savedlengthn);
        \fi
        \ifconnectsr
            \draw [reverse arrow style] (0,-0.5*\morphismheight-\savedlengths)
                to (0,-0.5*\morphismheight);
        \fi
    }
}
\pgfdeclareshape{swish right}
{
    \savedanchor\centerpoint
    {
        \pgf@x=0pt
        \pgf@y=0pt
    }
    \anchor{center}{\centerpoint}
    \anchorborder{\centerpoint}
    \anchor{north}
    {
        \pgf@x=\minimummorphismwidth
        \divide\pgf@x by 5
        \pgf@y=\morphismheight
        \divide\pgf@y by 2
        \advance\pgf@y by \connectheight
    }
    \anchor{south}
    {
        \pgf@x=-\minimummorphismwidth
        \divide\pgf@x by 5
        \pgf@y=-\morphismheight
        \divide\pgf@y by 2
        \advance\pgf@y by -\connectheight
    }
    \backgroundpath
    {
        \pgfsetstrokecolor{black}
        \pgfsetlinewidth{\thickness}
        \pgfpathmoveto{\pgfpoint
            {-0.2*\minimummorphismwidth}
            {-0.5*\morphismheight-\connectheight}}
        \pgfpathcurveto
            {\pgfpoint{-0.2*\minimummorphismwidth}{0pt}}
            {\pgfpoint{0.2*\minimummorphismwidth}{0pt}}
            {\pgfpoint
                {0.2*\minimummorphismwidth}
                {0.5*\morphismheight+\connectheight}}
        \pgfusepath{stroke}
    }
}
\pgfdeclareshape{swish left}
{
    \savedanchor\centerpoint
    {
        \pgf@x=0pt
        \pgf@y=0pt
    }
    \anchor{center}{\centerpoint}
    \anchorborder{\centerpoint}
    \anchor{north}
    {
        \pgf@x=-\minimummorphismwidth
        \divide\pgf@x by 5
        \pgf@y=\morphismheight
        \divide\pgf@y by 2
        \advance\pgf@y by \connectheight
    }
    \anchor{south}
    {
        \pgf@x=\minimummorphismwidth
        \divide\pgf@x by 5
        \pgf@y=-\morphismheight
        \divide\pgf@y by 2
        \advance\pgf@y by -\connectheight
    }
    \backgroundpath
    {
        \pgfsetstrokecolor{black}
        \pgfsetlinewidth{\thickness}
        \pgfpathmoveto{\pgfpoint
            {0.2*\minimummorphismwidth}
            {-0.5*\morphismheight-\connectheight}}
        \pgfpathcurveto
            {\pgfpoint{0.2*\minimummorphismwidth}{0pt}}
            {\pgfpoint{-0.2*\minimummorphismwidth}{0pt}}
            {\pgfpoint
                {-0.2*\minimummorphismwidth}
                {0.5*\morphismheight+\connectheight}}
        \pgfusepath{stroke}
    }
}
\pgfdeclareshape{state}
{
    \savedanchor\centerpoint
    {
        \pgf@x=0pt
        \pgf@y=0pt
    }
    \anchor{center}{\centerpoint}
    \anchorborder{\centerpoint}
    \saveddimen\overallwidth
    {
        \pgf@x=3\wd\pgfnodeparttextbox
        \ifdim\pgf@x<\minimumstatewidth
            \pgf@x=\minimumstatewidth
        \fi
    }
    \savedanchor{\upperrightcorner}
    {
        \pgf@x=.5\wd\pgfnodeparttextbox
        \pgf@y=.5\ht\pgfnodeparttextbox
        \advance\pgf@y by -.5\dp\pgfnodeparttextbox
    }
    \anchor{A}
    {
        \pgf@x=-\overallwidth
        \divide\pgf@x by 4
        \pgf@y=0pt
    }
    \anchor{B}
    {
        \pgf@x=\overallwidth
        \divide\pgf@x by 4
        \pgf@y=0pt
    }
    \anchor{text}
    {
        \upperrightcorner
        \pgf@x=-\pgf@x
        \ifhflip
            \pgf@y=-\pgf@y
            \advance\pgf@y by 0.4\stateheight
        \else
            \pgf@y=-\pgf@y
            \advance\pgf@y by -0.4\stateheight
        \fi
    }
    \backgroundpath
    {
        \pgfsetstrokecolor{black}
        \pgfsetlinewidth{\thickness}
        \pgfpathmoveto{\pgfpoint{-0.5*\overallwidth}{0}}
        \pgfpathlineto{\pgfpoint{0.5*\overallwidth}{0}}
        \ifhflip
            \pgfpathlineto{\pgfpoint{0}{\stateheight}}
        \else
            \pgfpathlineto{\pgfpoint{0}{-\stateheight}}
        \fi
        \pgfpathclose
        \ifblack
            \pgfsetfillcolor{black!50}
            \pgfusepath{fill,stroke}
        \else
            \pgfusepath{stroke}
        \fi
    }
}
\makeatother

\newcommand{\tinyhandle}[1][dot]{\raisebox{-2pt}{\ensuremath{\hspace{-3pt}\begin{pic}[scale=0.4,string]
        \node (0) at (0,0) {};
        \node[dot, inner sep=1.0pt] (1) at (0,0.3) {};
        \node[dot, inner sep=1.0pt] (2) at (0,0.7) {};
        \node (3) at (0,1) {};
        \draw (0.center) to (1.center);
        \draw (2.center) to (3.center);
        \draw[in=180, out=180, looseness=2] (1.center) to (2.center);
        \draw[in=0, out=0, looseness=2] (1.center) to (2.center);
\end{pic}\hspace{-1pt}}}}

\newcommand\cat[1]{\ensuremath{\mathbf{#1}}}

\newcommand\ket[1]{\ensuremath{| #1 \rangle}}
\newcommand\bra[1]{\ensuremath{\langle #1 |}}

\usepackage{color}

%% file: DensityEntailment-lncs.bbl
\begin{thebibliography}{10}

\bibitem{abramsky2004}
Samson Abramsky and Bob Coecke.
\newblock A categorical semantics of quantum protocols.
\newblock In {\em Proceedings of the 19th Annual IEEE Symposium on Logic in
  Computer Science}, pages 415--425. IEEE, 2004.

\bibitem{balkir2014}
Esma Balk{\i}r.
\newblock Using density matrices in a compositional distributional model of
  meaning.
\newblock Master's thesis, University of Oxford, 2014.

\bibitem{baroni2012}
Marco Baroni, Raffaella Bernardi, Ngoc-Quynh Do, and Chung-chieh Shan.
\newblock Entailment above the word level in distributional semantics.
\newblock In {\em Proceedings of the 13th Conference of the European Chapter of
  the Association for Computational Linguistics}, pages 23--32. Association for
  Computational Linguistics, 2012.

\bibitem{beltagy2013}
Islam Beltagy, Cuong Chau, Gemma Boleda, Dan Garrette, Katrin Erk, and Raymond
  Mooney.
\newblock Montague meets markov: Deep semantics with probabilistic logical
  form.
\newblock In {\em Second Joint Conference on Lexical and Computational
  Semantics}, volume~1, pages 11--21. Assocition of Computational Linguists,
  2013.

\bibitem{blacoe2013}
William Blacoe, Elham Kashefi, and Mirella Lapata.
\newblock A quantum-theoretic approach to distributional semantics.
\newblock In {\em Proceedings of the 2013 Conference of the North American
  Chapter of the Association for Computational Linguistics: Human Language
  Technologies}, pages 847--857, 2013.

\bibitem{bruza2005}
Peter~D Bruza and Richard Cole.
\newblock Quantum logic of semantic space: an exploratory investigation of
  context effects in practical reasoning.
\newblock {\em We Will Show Them: Essays in Honour of Dov Gabbay}, pages
  339--361, 2005.

\bibitem{clarke2009}
Daoud Clarke.
\newblock Context-theoretic semantics for natural language: an overview.
\newblock In {\em Proceedings of the Workshop on Geometrical Models of Natural
  Language Semantics}, pages 112--119. Association for Computational
  Linguistics, 2009.

\bibitem{coecke2010B}
Bob Coecke.
\newblock Quantum picturalism.
\newblock {\em Contemporary physics}, 51(1):59--83, 2010.

\bibitem{coecke2010}
Bob Coecke, Mehrnoosh Sadrzadeh, and Stephen Clark.
\newblock Mathematical foundations for a compositional distributional model of
  meaning.
\newblock {\em Linguistic Analysis}, 36, 2010.

\bibitem{firth1957}
John~R. Firth.
\newblock {A Synopsis of Linguistic Theory, 1930-1955}.
\newblock {\em Studies in Linguistic Analysis}, pages 1--32, 1957.

\bibitem{geffet2005}
Maayan Geffet and Ido Dagan.
\newblock The distributional inclusion hypotheses and lexical entailment.
\newblock In {\em Proceedings of the 43rd Annual Meeting on Association for
  Computational Linguistics}, pages 107--114. Association for Computational
  Linguistics, 2005.

\bibitem{glickman2005}
Oren Glickman, Ido Dagan, and Moshe Koppel.
\newblock Web based probabilistic textual entailment.
\newblock In {\em Proceedings of the PASCAL Challenges Workshop on Recognizing
  Textual Entailment}, 2005.

\bibitem{grefenstette2013}
Edward Grefenstette.
\newblock {\em Category-theoretic quantitative compositional distributional
  models of natural language Semantics}.
\newblock PhD thesis, University of Oxford, 2013.

\bibitem{grefenstette2011}
Edward Grefenstette and Mehrnoosh Sadrzadeh.
\newblock Experimental support for a categorical compositional distributional
  model of meaning.
\newblock In {\em Proceedings of the Conference on Empirical Methods in Natural
  Language Processing}, pages 1394--1404. Association for Computational
  Linguistics, 2011.

\bibitem{herbelot2013}
Aur{\'e}lie Herbelot and Mohan Ganesalingam.
\newblock Measuring semantic content in distributional vectors.
\newblock In {\em Proceedings of the 51st Annual Meeting of the Association for
  Computational Linguistics}, volume~2, pages 440--445. Association for
  Computational Linguistics, 2013.

\bibitem{kotlerman2010}
Lili Kotlerman, Ido Dagan, Idan Szpektor, and Maayan Zhitomirsky-Geffet.
\newblock Directional distributional similarity for lexical inference.
\newblock {\em Natural Language Engineering}, 16(04):359--389, 2010.

\bibitem{lambek2001}
Joachim Lambek.
\newblock Type grammars as pregroups.
\newblock {\em Grammars}, 4(1):21--39, 2001.

\bibitem{lenci2012}
Alessandro Lenci and Giulia Benotto.
\newblock Identifying hypernyms in distributional semantic spaces.
\newblock In {\em Proceedings of the First Joint Conference on Lexical and
  Computational Semantics}, volume~2, pages 75--79. Association for
  Computational Linguistics, 2012.

\bibitem{lin1998}
Dekang Lin.
\newblock An information-theoretic definition of similarity.
\newblock In {\em Proceedings of the International Conference on Machine
  Learning}, pages 296--304, 1998.

\bibitem{MacCartney2007}
Bill MacCartney and Christopher~D. Manning.
\newblock Natural logic for textual inference.
\newblock In {\em ACL Workshop on Textual Entailment and Paraphrasing}.
  Association for Computational Linguistics, 2007.

\bibitem{piedeleu2015}
Robin Piedeleu, Dimitri Kartsaklis, Bob Coecke, and Mehrnoosh Sadrzadeh.
\newblock Open system categorical quantum semantics in natural language
  processing.
\newblock {\em arXiv preprint arXiv:1502.00831}, 2015.

\bibitem{santus2014}
Enrico Santus, Alessandro Lenci, Qin Lu, and Sabine~Schulte Im~Walde.
\newblock Chasing hypernyms in vector spaces with entropy.
\newblock In {\em Proceedings of the 14th Conference of the European Chapter of
  the Association for Computational Linguistics}, volume~2, pages 38--42, 2014.

\bibitem{selinger2007}
Peter Selinger.
\newblock Dagger compact closed categories and completely positive maps.
\newblock {\em Electronic Notes in Theoretical Computer Science}, 170:139--163,
  2007.

\bibitem{sordoni2013}
Alessandro Sordoni, Jian-Yun Nie, and Yoshua Bengio.
\newblock Modeling term dependencies with quantum language models for ir.
\newblock In {\em Proceedings of the 36th International ACM SIGIR Conference on
  Research and Development in Information Retrieval}, pages 653--662.
  Association for Computational Linguistics, 2013.

\bibitem{van2004}
Cornelis~Joost Van~Rijsbergen.
\newblock {\em The geometry of information retrieval}.
\newblock Cambridge University Press, 2004.

\bibitem{weeds2004}
Julie Weeds, David Weir, and Diana McCarthy.
\newblock Characterising measures of lexical distributional similarity.
\newblock In {\em Proceedings of the 20th International Conference on
  Computational Linguistics}, number 1015. Association for Computational
  Linguistics, 2004.

\end{thebibliography}
